\documentclass{article}

\usepackage{microtype}
\usepackage{graphicx}
\usepackage{subcaption}
\usepackage{booktabs} 

\usepackage[hidelinks]{hyperref}
\usepackage{xcolor}

\definecolor{colordps}{HTML}{2E8B57}
\definecolor{colorembedopt}{HTML}{FF6347}

\newcommand{\colorDPS}[1]{%
  \begingroup\color{colordps}\ensuremath{#1}\endgroup
}

\newcommand{\colorEmbedopt}[1]{%
  \begingroup\color{colorembedopt}\ensuremath{#1}\endgroup
}

\usepackage[preprint]{neurips_2026}

\usepackage{amsmath, amssymb, amsfonts, amsthm, mathtools}
\usepackage{graphicx}
\usepackage[dvipsnames]{xcolor}
\usepackage{multirow}
\usepackage[shortlabels]{enumitem}
\usepackage{graphicx}
\usepackage{subcaption}
\usepackage{cancel}
\usepackage[sort,compress]{cleveref}%
\crefname{equation}{eq.}{eqs.}
\Crefname{equation}{Eq.}{Eqs.}
\Crefname{section}{\S}{\S}
\crefname{enumi}{Step}{Steps}

\usepackage{makecell} 

\usepackage{algorithm}
\usepackage{algorithmic}
\crefname{algocf}{algorithm}{algorithms}
\Crefname{algocf}{Algorithm}{Algorithms}

\DeclareMathOperator*{\argmax}{arg\,max}

\newtheorem{proposition}{Proposition}

\newtheorem{assumption}{Assumption}

\crefname{assumption}{assumption}{assumptions}
\Crefname{assumption}{Assumption}{Assumptions}
\crefname{definition}{definition}{definitions}
\Crefname{definition}{Definition}{Definitions}
\crefname{proposition}{proposition}{propositions}
\Crefname{proposition}{Proposition}{Propositions}

\newcommand{\BlackBox}{\rule{1.5ex}{1.5ex}}
\renewenvironment{proof}{\par\noindent{\bf Proof }}{\hfill\BlackBox\newline}

\definecolor{auxcolor}{RGB}{80,120,160}  

\newcommand{\smallo}{\mathop{o}}
\newcommand{\dd}{\mathop{}\!\mathrm{d}}

\newcommand{\Nres}{N_{\text{res}}}
\usepackage{enumitem}

\usepackage{zref-user}
\makeatletter
\zref@newprop{algline}{\theALC@line}
\zref@addprop{main}{algline}
\makeatother

\newcommand{\linelabel}[1]{\zlabel{#1}}
\newcommand{\lineref}[1]{Line~\zref@extractdefault{#1}{algline}{??}}

\usepackage[textsize=tiny]{todonotes}

\title{Robust Inference-Time Steering of Protein Diffusion Models via Embedding Optimization}

\author{%
  Minhuan Li\thanks{Correspondence to: Minhuan Li \texttt{<minhuanli@flatironinstitute.org>} and Luhuan Wu \texttt{<lwu@flatironinstitute.org>}.} \\
  Flatiron Institute \\
  \And
  Jiequn Han \\
  Flatiron Institute \\
  \AND
  Pilar Cossio \\
  Flatiron Institute \\
  \And
  Luhuan Wu\footnotemark[1] \\
  Flatiron Institute \\
}

\begin{document}

\maketitle

\begin{abstract}
A core challenge in structural biophysics is generating biomolecular conformations that are both physically plausible and consistent with experimental measurements. While sequence-to-structure diffusion models provide powerful priors, posterior sampling methods steer generation by perturbing atomic coordinates with gradients from experimental likelihoods. However, when the target lies in a low-density region of the prior, these methods require aggressive upweighting of the likelihood that can destabilize sampling and be sensitive to hyperparameters. We propose EmbedOpt, an inference-time steering framework that introduces an orthogonal optimization axis: rather than performing posterior sampling under a fixed prior, EmbedOpt directly optimizes the prior by updating the model's conditional embedding. This embedding space encodes rich coevolutionary signals, so optimizing it shifts the structural prior to align with experimental constraints. Empirically, EmbedOpt matches coordinate-based posterior sampling baselines on sparse distance constraints and outperforms them on cryo-electron microscopy map fitting, including real, noisy experimental ones. Furthermore, EmbedOpt's smooth optimization behavior yields robustness to hyperparameters spanning two orders of magnitude and enables comparable performance with fewer diffusion steps. Code is available at \url{https://github.com/rs-station/embedopt}.

\end{abstract}

\section{Introduction}

Biomolecules are dynamic systems that undergo conformational changes in order to perform their biological functions. Consequently, identifying 
relevant conformations
is essential for understanding its functional mechanisms. However, inferring them from experimental data is a challenging inverse problem. This difficulty arises from the high dimensionality of conformational space and  the need to enforce physical realism while maintaining agreement with experimental data. Traditionally, this problem has been addressed using physics-based force fields with stepwise sampling under the maximum entropy principle \cite{rieping2005inferential}; although powerful, such methods are computationally demanding and require extensive refinement to match both the physics and data \cite{croll2018isolde}.

A recent paradigm shift in the field leverages pretrained generative models as data-driven priors to regularize the conformational search. Protein sequence-to-structure models such as Alphafold~3 \citep{abramson2024accurate}  use a conditional diffusion module \citep{song2020score,ho2020denoising}
which enables \textit{distributional} modeling of 3D coordinates $x_0$ conditioned on sequence information $c$ by iteratively refining noisy structures $x_t$. 
While these models capture dominant structural modes with remarkable accuracy,  structure-determination inverse problems often require recovering  conformational states that lie outside the high-probability regions of the prior model. This necessitates the incorporation of additional constraints—e.g., derived from experimental measurements—that were not present during training. Viewed through a Bayesian lens, this is framed as a posterior inference problem:  the pretrained model defines a sequence-conditioned prior $p(x_0 \mid c)$ while experimental measurements $y$ define a likelihood $p(y \mid x_0)$,  yielding a posterior distribution $p(x_0 \mid y,c) \propto p(x_0 \mid c) p(y \mid x_0)$. 

\begin{figure}[!t]
    \centering
    \includegraphics[width=\linewidth]{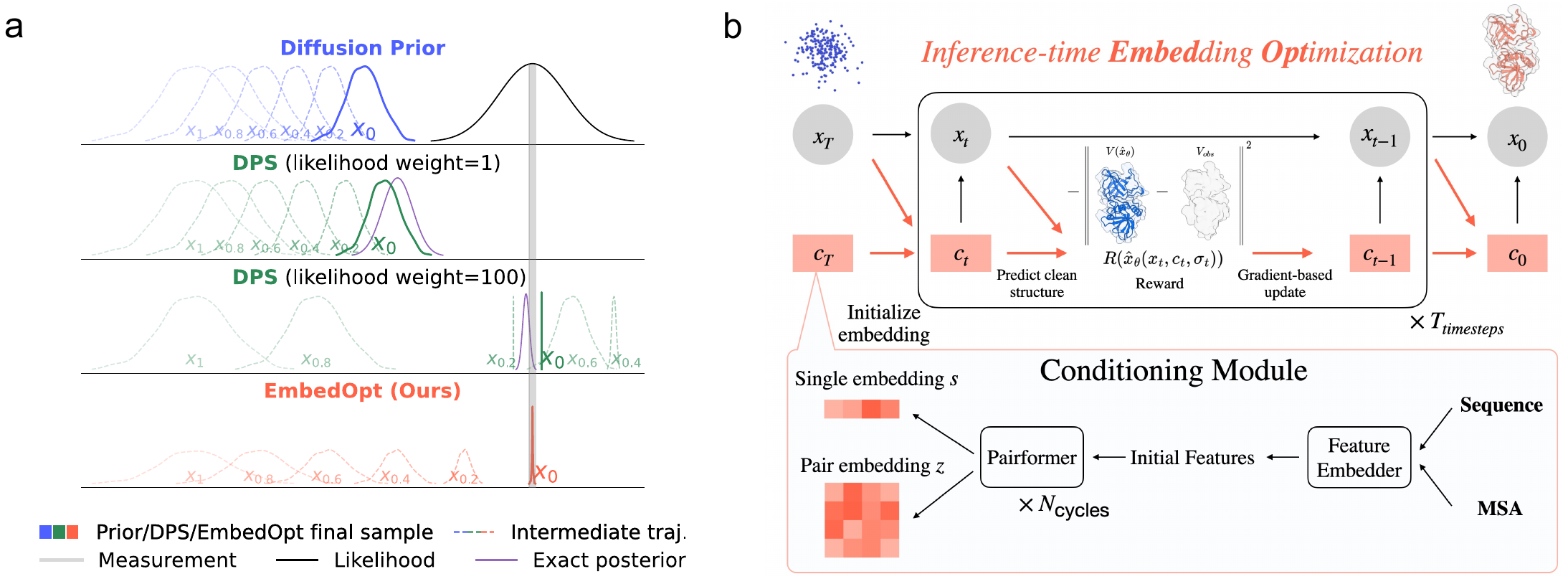}
        
    \caption{\textbf{(a) Illustration of DPS vs.\ EmbedOpt under prior--measurement mismatch.} The diffusion prior has limited overlap with the measurement (\textit{top}), so both the exact posterior and DPS samples remain distant from it (\textit{second}). Upweighting the likelihood pulls DPS samples closer but produces an ill-conditioned sampling landscape (\textit{third}). EmbedOpt (\textit{bottom}) instead iteratively updates the conditional embedding, dynamically shifting the prior to produce samples concentrating around the measurement (see \Cref{app:sec:synthetic}).
\textbf{(b) Schematic of EmbedOpt.} EmbedOpt adapts a pretrained diffusion model at inference time through optimization of the conditional embedding to maximize an experimental likelihood reward $R(\cdot)$ . \textit{(Top)} In a single forward pass, the  embedding is updated $c_t \to c_{t-1}$ using gradients from the denoised structure $\hat{x}_\theta$, greedily increasing the reward. \textit{(Bottom)} We instantiate EmbedOpt on an AlphaFold 3-style backbone, where the Conditioning Module embeds sequence and MSA inputs into the initial embedding $c_T$.}
    \label{fig:overview}
\end{figure}

Because the exact diffusion posterior is intractable, Diffusion Posterior Sampling (DPS) \citep{chung2022diffusion} and many subsequent methods provide a practical approximation by applying likelihood-gradient guidance to intermediate noisy atomic coordinates $x_t$, progressively steering samples toward high-likelihood regions. DPS-style approaches have been applied to recover protein conformations from cryo-electron microscopy (cryo-EM) maps and NOE spectroscopy data \citep{maddipatla2025inverse,raghumultiscale}, but their performance depends sensitively on the likelihood weighting, a phenomenon also observed in image inverse problems \citep{he2023manifold,zach2025statistical}. This balance is delicate, especially under prior--likelihood mismatch: as illustrated in \Cref{fig:overview}(a), when the prior assigns little mass to high-likelihood regions, DPS fails to generate measurement-consistent samples without aggressive likelihood upweighting. 
Similar behaviors arise in real systems: \citet{raghumultiscale} observe that models like AlphaFold 3 and Boltz-1 \citep{wohlwend2025boltz}  produce
samples concentrated around a dominant conformation and report limited optimization stability for their DPS-based method

To overcome these limitations, we propose \textit{EmbedOpt}, an inference-time method that reduces prior--likelihood mismatch by optimizing the conditional embedding $c$, effectively reshaping the diffusion prior to align with experimental constraints (\Cref{fig:overview}). This opens embedding-space optimization as a new algorithmic axis for inference-time steering of protein diffusion models, complementary to the predominant coordinate-space axis. In AlphaFold~3-style architectures, conditional embedding encodes rich coevolutionary signal derived from multiple sequence alignments (MSAs), so modifying it reshapes the model's structural preferences without altering its parameters. EmbedOpt formulates inverse structure determination as reward maximization and performs lightweight, single-step gradient ascent on the embedding at each diffusion step, preserving the model's structural regularization at substantially lower cost than fine-tuning.

We evaluate EmbedOpt against DPS-based methods on three biomolecular inverse problems: sparse residue-pair distance constraints (24 multi-domain proteins), synthetic cryo-EM map fitting (77 systems), and real experimental cryo-EM map fitting (6 targets from the CryoBoltz benchmark, where we additionally compare against the published CryoBoltz baseline). Across all three, EmbedOpt exhibits stable optimization across two orders of magnitude of learning rate, requires  fewer diffusion steps, and matches or outperforms baselines on real targets while preserving stereochemical quality.

\paragraph{Contributions.} We propose EmbedOpt to align  protein diffusion models with experimental measurement, with the following contributions: 
\textbf{(i) A new algorithmic axis} for inference-time steering of protein diffusion models, optimizing the conditional sequence embedding. 
\textbf{(ii) Theoretical analysis}: a trust-region characterization that establishes monotonic optimization behavior of EmbedOpt under small-step updates.  
\textbf{(iii) Robustness and efficiency}: stable performance across a wider hyperparameter range and with fewer diffusion steps than DPS-based methods, while preserving performance and stereochemistry on both distance-constraint and cryo-EM benchmarks.
\textbf{(iv) Real-world validation}: EmbedOpt matches or outperforms the published CryoBoltz baseline on 5/6 experimental cryo-EM targets while producing substantially better stereochemical geometry.

\section{Background}
\paragraph{Diffusion Models.}
Diffusion models \citep{ho2020denoising,song2020score} are powerful generative models that sample from a target distribution $p(x_0)$ through an iterative process.  We consider the probability flow ODE formulation \citep{karras2022elucidating}
\begin{align}\label{eq:prob-flow-ode}
    \dd x_t = - \dot{\sigma}(t) \sigma(t) \nabla_{x_t} \log p_t(x_t) \dd t 
\end{align}
where $\sigma(t)$ is a noise level increasing in $t$ with $\sigma(0)=0$,  $p_t(x_t) = \int p(x_0) \mathcal{N}(x_t | x_0, \sigma(t)^2)\dd x_0 $ is the time-dependent marginal distribution, and  $\nabla_{x_t} \log p_t(x_t)$ is the score function. The idea of diffusion models is that 
starting from $x_T \sim \mathcal{N}(0, \sigma_T^2\mathbb{I})$ for some large $\sigma_T$,  integrating this ODE backward in time recovers  target samples $x_0 \sim p(x_0)$. 
While the score is typically intractable, it can be approximated using a \textit{denoiser} network $\hat x_\theta(x_t, \sigma(t))$ trained to predict the posterior mean $\mathbb{E}[x_0 \mid x_t]$. Tweedie's formula \citep{efron2011tweedie} yields an approximation of the score $
\nabla_{x_t} \log p_t(x_t) \approx  [\hat x_\theta \left(x_t, \sigma(t)\right) -x_t] / {\sigma^2(t)}.
$
With a learned denoiser, we numerically integrate \Cref{eq:prob-flow-ode} backward on a discrete time grid $\tau_T>  \cdots > \tau_0=0$.  With a slight abuse of notation, we write $x_t := x_{\tau_t}$ and $\sigma_t:=\sigma(\tau_t)$ for  $t \in \{0,\dots,T\}$.  An Euler step from $\tau_t$ to $\tau_{t-1}$ gives
\begin{align}
\begin{split}
x_{t-1}&\approx x_t + \eta_t\bigl(\hat x_\theta(x_t,\sigma_t)-x_t\bigr), \quad \textrm{ where } \eta_t  \coloneqq {\Delta \sigma_t}/{\sigma_t} = {(\sigma_t-\sigma_{t-1})}/{\sigma_t}.   \label{eq:diffusion-update}
\end{split}
\end{align}

\paragraph{Protein Sequence-to-Structure Diffusion Models.}

Recent protein sequence-to-structure prediction models like AlphaFold~3 introduce a conditional diffusion head to predict a distribution $p(x_0 \mid c)$ of structures $x_0$ from sequence information $c$. This conditional diffusion 
is realized by designing the denoiser $\hat{x}_\theta$ in \Cref{eq:diffusion-update} to accept additional input $c$ such that it is written as $\hat{x}_\theta(x_t, c, \sigma_t)$. 
As illustrated in \Cref{fig:overview}(b), 
conditional embeddings are computed from the sequence and MSA information, which encodes coevolutionary signals. 
Conditioning the diffusion process on these embeddings anchors a strongly informed prior over protein structures, ensuring that generated structures remain consistent with the evolutionary constraints of the target sequence. 
In this work, we focus on the conditioning inputs $c=\{s, z\}$, where the \textit{single} embedding $s$ and \textit{pair} embedding $z$ are produced by the PairFormer module and capture \textit{residue-wise} and \textit{residue-pair} interactions, respectively.

\section{Method}

\paragraph{Setup.} 
We consider a conditional protein diffusion model that defines a prior distribution over 3D structures $x_0 \in \mathbb{R}^{\Nres \times 3}$,  conditioned on sequence embeddings $c=\{s,z\}$, where single embedding $s\in \mathbb{R}^{\Nres \times C_s}$ and pair embedding $ z \in \mathbb{R}^{\Nres \times \Nres \times C_z}$. Here $\Nres$ denotes the number of residues in the sequence, and $C_s$ and $C_z$  denote the channel dimensions of single and pair embeddings. 
For many inverse problems, we can prescribe an experimental likelihood  $p(y \mid x_0)$ given some measurement $y$ of the underlying structure $x_0$. We define a reward function $R(x_0) \propto \log p(y\mid x_0)$\footnote{Since our goal is to maximize the reward, scaling the log likelihood does not influence the optimum.} and for simplicity we omit the dependence on $y$. We assume $R$ is differentiable, as is the case in many applications. Our goal is to generate structures $x_0$ with high reward while maintaining physical plausibility.

\paragraph{EmbedOpt: Objective and Main Algorithm.}\label{subsec:method:embedopt}
Given an initial sample $x_T$ and embedding $c$, the diffusion model produces a final sample $x_0(x_T, c)$ as a function of $x_T$ and $c$, which can be stochastic due to the additional noise injection in the sampling scheme. For a fixed $x_T$, the \textit{oracle global objective} to optimize the embedding $c$ is 
  $  \argmax_{c} R(x_0(x_T, c))$. 
Directly optimizing this objective via gradient-based methods would require backpropagating through the entire sampling trajectory, which can be memory-intensive or numerically unstable.
To address this challenge, we adopt a dynamic optimization strategy over the sampling path. Given an intermediate sample $x_t$ at step $t$, we consider a  \textit{surrogate local objective} $ \argmax_{c} R(\hat x_\theta(x_t, c, \sigma_t))$. 
That is, we optimize the reward evaluated at the denoising prediction given  sample $x_t$ and noise level $\sigma_t$, which we call the \textit{surrogate reward}. As  $\sigma_t \to 0$ the surrogate reward is closer to the true reward. 
This relaxation is in the same spirit of 
diffusion posterior sampling works \citep[e.g.][]{chung2022diffusion}. 

\begin{algorithm}[t]
\caption{EmbedOpt}
\label{alg:embed-opt}
\begin{algorithmic}[1] 
\STATE  \textbf{Input:} differentiable reward function $R$,  original conditional embedding $c$, denoiser network $\hat x_\theta$,  adaptive learning rate $\alpha_t$, total steps $T$,  noise  schedule $\{\sigma_t\}_{t=0}^T$\\
\STATE  \textbf{Output:} final sample $x_0$

\STATE Initialize embedding $c_T \gets c$
and sample initial coordinate $x_T \sim \mathcal{N}(0,\sigma_T^2)$
\FOR{$t = T, T-1, \dots, 1$}
  \STATE\linelabel{alg:line:denoising-prediction}
   Make denoised prediction $\hat{x}_0 \gets \hat{x}_\theta(x_t, \colorEmbedopt{c_t}, \sigma_t)$ 
   
   \STATE \linelabel{alg:line:embedupdate} Update embedding $\colorEmbedopt{c_{t-1} \gets c_t + \alpha_t \nabla_{c_t} R(\hat{x}_0)}$ 
  \STATE  \linelabel{alg:line:coordupdate} Update  coordinate $x_{t-1} \gets x_t + \eta_t \big[\hat x_\theta(x_t, \colorEmbedopt{c_{t-1}}, \sigma_t)-x_t\big]$ where    $\eta_t =\frac{ \sigma_t-\sigma_{t-1}}{\sigma_{t-1}}$
\ENDFOR

\end{algorithmic}
\end{algorithm}

In practice, we perform greedy inference-time optimization by taking a single gradient step on the surrogate reward at each diffusion step. \Cref{alg:embed-opt} summarizes the procedure. We initialize the embedding $c_T =c$ with the pretrained model's conditioning inputs. At each step, we use the gradient of the surrogate reward to update embedding from $c_t$ to $c_{t-1}$. The updated embedding is then used both to advance the diffusion process  and to initialize optimization at the next step. This dynamically adapts the prior model while steering the sampling trajectory toward higher surrogate-reward regions.

While \Cref{alg:embed-opt} assumes a generic adaptive learning rate $\alpha_t$, in practice we find normalizing the gradient by its root mean squared value and using a constant base learning rate $\alpha$ is  effective and simplifies tuning. We apply the gradient normalization for the single and pair embeddings $s_t$ and $z_t$: 
\[
\bar g_{s_t} = \nabla_{s_t} R(\hat x_0) / \texttt{RMS} (\nabla_{s_t} R(\hat x_0)), \qquad \bar g_{z_t} = \nabla_{z_t} R(\hat x_0) / \texttt{RMS} (\nabla_{z_t} R(\hat x_0))
\]

where $\texttt{RMS}(v):=\sqrt{\frac{1}{d} \sum_{i=1}^d v_i^2}$ for a $d$-dimensional vector $v$. 
Then embeddings are updated with 
\begin{align}\label{eq:embedopt-update-normalized}
    s_{t-1} 
= s_t \colorEmbedopt{+ \alpha \, \bar g_{s_t}}~,
\qquad z_{t-1}
= z_t + \colorEmbedopt{\alpha \, \bar g_{z_t}}~. 
\end{align}
The complete algorithm with this RMS normalized gradient update is given in \Cref{alg:embed-opt-stochastic}, which is also adapted to the stochastic sampling scheme used in AlphaFold~3. 

\paragraph{Theoretical Analysis.}\label{subsec:method:analysis}
Although EmbedOpt operates as a dynamic optimization procedure, we show that it yields monotonic improvement of the surrogate reward under suitable regularity and trust-region assumptions, which can be achieved by choosing sufficiently small step size $\Delta \sigma_t$ and learning rate $\alpha$. 
Denote $F(x,c,\sigma) \coloneqq R(\hat x_\theta(x,c,\sigma))$ the surrogate reward  given $(x,c,\sigma)$. The following guarantee of the surrogate reward improvement for EmbedOpt holds:
\begin{proposition}[Informal]\label{prop:cross-time-minimal-informal} 
For a fixed $(x_t, c_t, \sigma_t)$, consider the transition to $(x_{t-1}, c_{t-1}, \sigma_{t-1})$ under the  EmbedOpt update (\Cref{alg:embed-opt}, Line 6-7). 
Assume the surrogate reward $F$ is locally smooth and $\hat x_\theta$ is bounded in a neighborhood of $(x_t, c_t, \sigma_t)$. The following bound holds 
\begin{align}
\begin{split}
F(x_{t-1},c_{t-1}, \sigma_{t-1})  \ge 
F(x_t,c_t, \sigma_t) +
\frac{\alpha_t}{2}\big\|g_{c_t}\big\|^2 &- \Delta \sigma_t \left[ \frac{G_x}{\sigma_t}  \|\hat x_\theta (x_t, c_{t-1}, \sigma_t)-  x_t\| 
+
G_\sigma \right]
\label{eq:gap-min-informal}
\end{split}
\end{align}
where $g_{c_t}=\nabla_{c_t} F(x_t,c_t, \sigma_t)$ and $G_x$ and $G_\sigma$ are Lipschitz constants depending on $(x_t, c_t, \sigma_t)$. 

Moreover, if the learning rate $0 \leq \alpha_t \leq \alpha_{\max}$ and the noise level step size $0 \leq \Delta \sigma_t \leq \Delta_{\max}$, 
where constants $\alpha_{\max}, \Delta_{\max}$ depend on $(x_t, c_t, \sigma_t)$,  
 the surrogate reward is non-decreasing: $F(x_{t-1},c_{t-1},\sigma_{t-1}) \geq F(x_t,c_t,\sigma_t)$  
(see a formal statement and proof in \Cref{app:subsec:prop}). 
\end{proposition}

This proposition shows that EmbedOpt acts as a local trust-region method: within a sufficiently small trust region, the gradient-ascent gain $\tfrac{\alpha_t}{2}\|g_{c_t}\|^2$ in \Cref{eq:gap-min-informal} dominates the deviation from coordinate and noise-level updates, yielding a non-decreasing surrogate reward. When $g_{c_t}=0$, both the first-order improvement and the trust-region condition vanish and EmbedOpt remains at a local optimum.

\paragraph{Comparison to DPS.}\label{subsec:method:comparison}
DPS \citep{chung2022diffusion}  steers generation using  likelihood-gradient guidance on the noisy coordinate $x_t$ under a fixed diffusion prior $p(x_0\mid c)$ (full background in \Cref{app:subsec:dps-background}). To make a structural comparison, we rewrite per-step updates of both methods in a common form that reduces to a denoising step plus a reward-gradient term 
(derivation in \Cref{app:subsec:taylor})
\begin{align}
\textrm{DPS:}\quad     & x_{t-1} = x_t + \eta_t \bigl(\hat x_0 - x_t \;\colorDPS{+\; \alpha_t\, J_{x_t}^\top \nabla_{\hat x_0} R(\hat x_0)}\bigr), \label{eq:dps-update-main-text}\\
\textrm{EmbedOpt:}\quad & x_{t-1} \approx x_t + \eta_t \bigl[\hat x_0 - x_t \;\colorEmbedopt{+\; \alpha_t\, J_{c_t} J_{c_t}^\top \nabla_{\hat x_0} R(\hat x_0)}\bigr], \label{eq:embedopt-update-taylor}
\end{align}
where $\hat x_0 \coloneqq \hat x_\theta(x_t, c_t, \sigma_t)$, and $J_{x_t}$ and $J_{c_t}$ denote the Jacobians of $\hat x_0$ with respect to $x_t$ and $c_t$, respectively.
Both updates share the same \emph{denoising-reward gradient} $\nabla_{\hat x_0} R(\hat x_0)$, but transform it through different operators. DPS pulls this gradient back to the noisy-coordinate space via $J_{x_t}^\top$, coupling the update direction to the local sensitivity of the denoiser with respect to $x_t$. EmbedOpt instead applies a $J_{c_t} J_{c_t}^\top$ preconditioner: the coordinate update direction stays within the span of the reward gradient, while its magnitude and anisotropy are modulated by the embedding-space geometry. We show empirically in \Cref{sec:experiments} that this distinction translates to  greater robustness of EmbedOpt.  

\section{Related Works}\label{sec:related-works}

\paragraph{Diffusion-based Methods for Solving Structure-Determination Inverse Problems.}
\citet{maddipatla2025inverse} and \citet{raghumultiscale}  address biophysical inverse problems as a posterior sampling task under a AlphaFold~3-style  pretrained prior model and operate within the DPS  framework. \citet{maddipatla2025inverse} infer conformational ensembles, i.e. multiple conformational states corresponding to the same protein sequence, from ensemble-averaged experiment data, with experiments restricted to X-ray and NMR data; their pipeline introduces a gradient-normalization schedule  which we adopt verbatim in our DPS baseline (\Cref{subsec:exp:real-cryoem}).  
CryoBoltz \citep{raghumultiscale} focuses on the cryo-EM application, introducing a multi-scale DPS strategy which applies global structural constraints in the early sampling stage and local constraints in later stage; we benchmark  against  CryoBoltz  on real cryo-EM targets with EmbedOpt using a uniform objective throughout optimization  (\Cref{subsec:exp:real-cryoem}). 
Our work is orthogonal to coordinate-based methods, unlocking a new  axis for biophysical inverse problems. 

A concurrent study, \citet{maddipatla2026inference}, also explores inference-time embedding-space optimization, applying it to ensemble structure generation through a multi-round procedure that repeatedly optimizes embeddings across diffusion runs.
Our work is complementary in scope: we develop embedding-space steering as a general inference-time framework that updates the embedding within a single diffusion trajectory, with a trust-region analysis and empirical comparison against coordinate-space methods on single-structure determination, showing improved robustness; \citet{maddipatla2026inference} focus on empirical application to ensemble generation with task-specific design choices. 

\paragraph{Exploration of Alternative Conformations.}

An adjacent line of research is to sample diverse conformational states of protein structures. While existing sequence-to-structure deep learning models exhibit limited diversity, an exploratory line of research seeks to simulate alternative conformations by subsampling or perturbing the MSA inputs \citep{wayment2024predicting,kalakoti2025afsample2,lee2025large,del2022sampling}. More recently, \citet{richman2025unlocking} propose an inference-time diffusion sampling algorithm in Alphafold~3-style models that generate conformations as a prescribed distance away from a reference structure, using a combination of DPS and particle filtering methods \citep{wu2023practical}. 
Our current work focuses on generating conformations consistent with experimental constraints by adapting a pretrained conditional diffusion model. 

We summarize non-diffusion-based methods for structure determination and general inference-time steering methods for diffusion models (e.g., for image applications) in \Cref{app:sec:additional-related-works}.

\section{Experiments}\label{sec:experiments}

 We evaluate EmbedOpt against baselines, on two structure determination tasks: (i) satisfying sparse residue-pair distance constraints, a controlled proxy for data from techniques like chemical cross-linking or FRET (\Cref{subsec:exp:distance}); and (ii) fitting to synthetic (\Cref{subsec:exp:synthetic-cryoem}) and real (\Cref{subsec:exp:real-cryoem}) experimental cryo-EM density maps. Both tasks probe challenging regimes where experimental targets lie far from the pretrained model's dominant mode. We show that EmbedOpt: \textbf{(i)} matches or exceeds DPS in best-achieved performance, \textbf{(ii)} is robust to learning rates spanning two orders of magnitude where DPS is more sensitive, \textbf{(iii)} permits a $4\times$ reduction in diffusion steps without performance loss, and \textbf{(iv)} outperforms the published CryoBoltz baseline on 5 out of 6 real cryo-EM targets.

\paragraph{Experiment Protocol.} 
Our experiments use the Protenix model \citet{bytedance2025protenix} as the diffusion prior. We compare EmbedOpt to the unguided prior and to DPS within the same model; on real cryo-EM data we additionally compare to CryoBoltz \cite{raghumultiscale}, a published DPS-based method built on Boltz-1 \cite{wohlwend2025boltz}. We apply the gradient normalization strategy from \citet{maddipatla2025inverse}, which is also based on Protenix, to DPS (see \Cref{alg:dps-alphafold} for full details). This places the two methods in a relatively comparable optimization regime parameterized by a single base learning rate $\alpha$. All generated structures undergo energy relaxation, following \citet{maddipatla2025inverse}; this step is lightweight and improves local geometry with only mild degradation in reward (\Cref{app:sec:energy_relaxation}). All runs use 3 random seeds unless otherwise specified.
In terms of computational cost, 
EmbedOpt requires roughly $1.4$–$1.8\times$ the per-diffusion-step compute of DPS, reflecting one extra forward pass through the denoiser; the end-to-end inference compute gap shrinks as the pre-diffusion costs are included; both methods remain practical on a single GPU at the system sizes we consider (see \Cref{app:sec:runtime_profiling} for compute profiling). All experiment details are in  \Cref{app:sec:experiment}.

\subsection{Distance-Constrained Structure Determination}\label{subsec:exp:distance}

\begin{figure}[!t]
    \centering
    \includegraphics[width=1.0\linewidth]{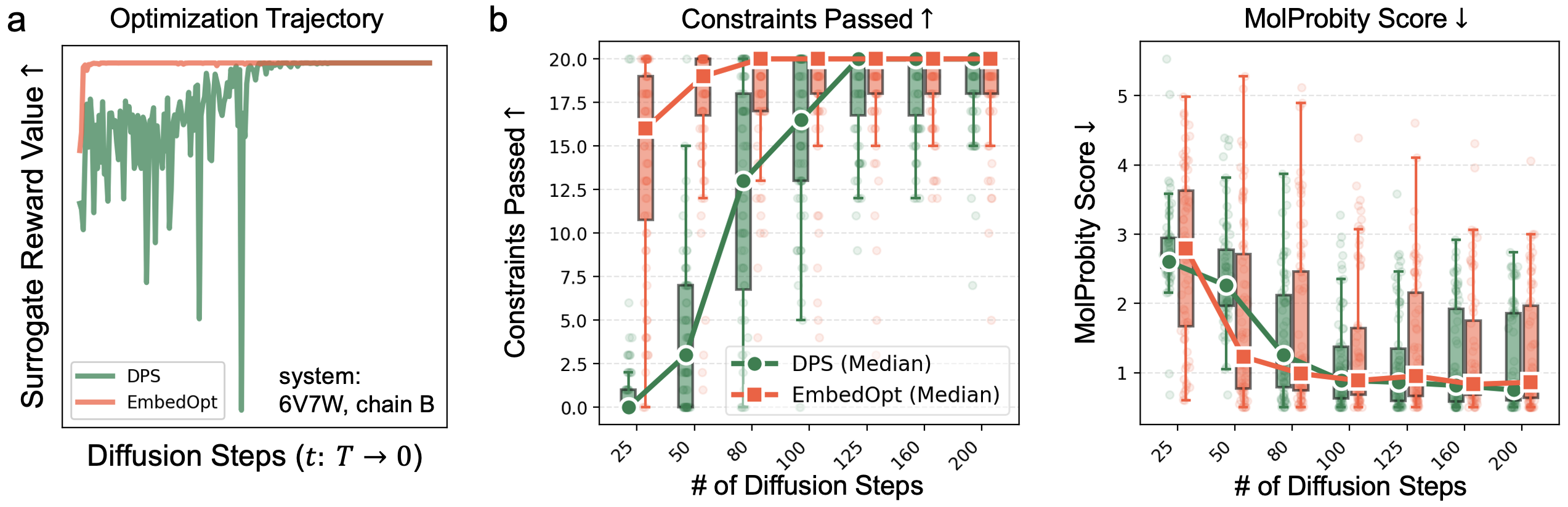}
    \caption{\textbf{Distance Constraint Benchmark.} (a) \textbf{Optimization Trajectory:} Representative reward traces show a smooth and monotonic EmbedOpt increase, while DPS has high-frequency volatility. (b) 
    \textbf{Step-Efficiency Scaling:}
    Constraints passed (left) and MolProbity score (right) as a function of the number of diffusion steps. DPS performance deteriorates sharply below 100 steps, whereas EmbedOpt maintains a median constraint satisfaction rate above 75\% and lower MolProbity scores.
    }
    \label{fig:distance_constraint}
\end{figure}

Experimental techniques such as chemical cross-linking or FRET often yield sparse distance constraints. To evaluate inference-time steering under such conditions, we use residue-pair distance constraints as a controlled proxy. We assemble a benchmark of 24 multi-domain protein systems from \citet{zhang2025distance}, as they exhibit flexible inter-domain orientations that challenge existing generative models. To control the task difficulty, we derive constraints by selecting the top $K=20$ residue pairs (a sparsity level typical of cross-linking experiments) with the largest structural discrepancies between the ground truth and the prior model's predictions, with a $\delta = 2.0$ \AA\ tolerance. We note that this constraint-selection procedure uses oracle access to the target and is therefore an idealized stress-test rather than a model of a specific experimental protocol; in practice, constraint pairs would be specified by the experimental modality (e.g., cross-linker reactivity profiles). The reward function is
$R(x_0) = - \sum_{i=1}^K \min\!\left(|d_i(x_0) - d_i^{\textrm{target}}|, \, \delta\right)^2$,
where $d_i(x_0)$ is the predicted distance for the $i$-th constrained pair, and $d_i^{\textrm{target}}$ is the target distance. Performance is measured by the number of satisfied constraints; stereochemical quality is measured by the MolProbity score.

\begin{figure}[!t]
    \centering
    \includegraphics[width=1.0\linewidth]{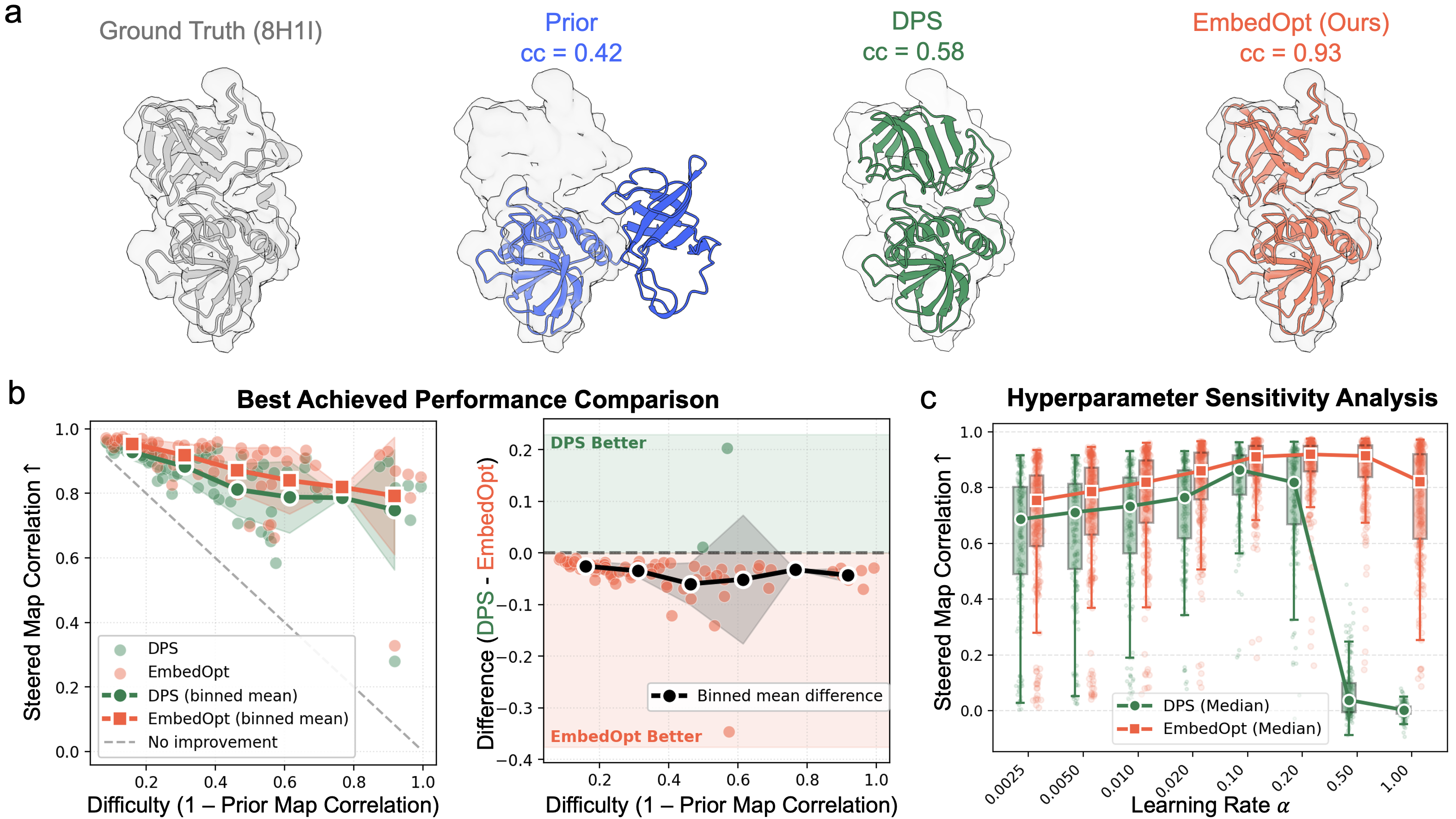}
    \caption{\textbf{Cryo-EM Map Fitting Benchmark.} (a)\textbf{ Visualization of a challenging target:} 8H1I requires significant inter-domain rearrangement of the prior structure (correlation coefficient - $cc = 0.42$) to fit the target(gray volume). DPS remains trapped in a local optimum ($cc=0.58$), while EmbedOpt successfully reorients the domains ($c c=0.93$). (b) \textbf{Best-achieved Performance vs. Task Difficulty}: (left) Best-sampled structures across 77 systems (dots), binned by task difficulty ($1 - $ prior map correlation). EmbedOpt maintains an advantage especially on harder targets ($>0.4$). (right) The difference plot shows EmbedOpt outperforming DPS across the majority of systems.
    (c) \textbf{Hyperparameter Sensitivity:} Distribution of map correlations across all systems for varying learning rates, spanning the interquartile range (25th-75th percentile) with the median highlighted. EmbedOpt maintains high performance across learning rates spanning two orders of magnitude.
    }
    \label{fig:map_reward}
\end{figure}

\textbf{Smoother Reward Optimization Trace.} 
Our theoretical analysis (\Cref{prop:cross-time-minimal-informal}) suggests that EmbedOpt's preconditioned update should operate within a well-behaved trust region, potentially avoiding the geometric sensitivities that affect DPS. To empirically verify this, we analyze the optimization dynamics under $T=200$ diffusion steps. \Cref{fig:distance_constraint}a reveals a sharp contrast consistent with our expectations: while DPS's reward trace fluctuates frequently,  EmbedOpt exhibits a smooth, monotonic improvement of the surrogate reward. We observe this stability across test systems and hypothesize that it translates into practical efficiency by permitting larger optimization step sizes.

\textbf{Robust Performance with Reduced Diffusion Steps.} 
To test this hypothesis, we perform an efficiency scaling experiment on the distance constraint task. Starting from a 200-step baseline (with $\alpha = 0.1$, optimal for both methods per \Cref{app:fig:af_distance_results}b), we reduce the diffusion steps while proportionally increasing learning rates to maintain a constant total guidance magnitude ($\alpha \times $ \# steps $=$ const). As shown in \Cref{fig:distance_constraint}b, EmbedOpt demonstrates remarkable resilience: it maintains near-optimal constraint satisfaction (median constraints passed $\approx 20$) down to 50 diffusion steps—a $4\times$ speedup over the baseline—while preserving excellent MolProbity scores ($<2.0$). DPS, in contrast, collapses below 100 steps as high learning rates destabilize coordinate updates.

With sufficient steps and tuned hyperparameters, both methods achieve comparable peak performance (\Cref{app:fig:af_distance_results}a), as the flexibility of sparse distance targets allows both  to saturate the reward. However, EmbedOpt exhibits broader hyperparameter robustness, maintaining valid geometries across a much wider learning-rate spectrum (\Cref{app:fig:af_distance_results}b). \emph{Takeaway: EmbedOpt achieves $4\times$ reduction in diffusion steps over DPS while preserving constraint satisfaction and stereochemical validity.}

\subsection{Cryo-EM Map Fitting (Synthetic Benchmark)}\label{subsec:exp:synthetic-cryoem}

In single-particle cryo-EM, 3D density maps are reconstructed from filtered 2D projection images. A typical downstream task is to build an atomic model that maximizes agreement with the map, often starting from an existing PDB deposition or a prediction model. To simulate this problem at scale, we assemble a diverse benchmark of 77 protein systems from the PDB with high-resolution structures ($<$4.5 \AA) and low sequence similarity ($<$25\%). We generate synthetic target cryo-EM maps $V_{\textrm{obs}}$ at 5.0 \AA\ resolution using \texttt{SFC\_Torch} \cite{li2025sfcalculator}. To accurately reflect cryo-EM physics, the forward model explicitly simulates electrostatic potentials using electron scattering form factors, following standard practice in cryo-EM validation suites \cite{afonine2018new} (see \Cref{app:sec:cryoem} for details). Atomic B-factors are uniformly set to 50 \AA$^2$ to standardize thermal variation. We use the same differentiable forward model to render a map $V(x_0)$ from a generated structure $x_0$. 
The reward is the negative mean squared error between the normalized target map and the normalized rendered map
$R(x_0) = - \frac{1}{N_x N_y N_z} \sum_{N_x, N_y, N_z}\left(V(x_0) - V_{\textrm{obs}}\right)^2,$
where $V(x_0), V_{\textrm{obs}} \in \mathbb{R}^{N_x\times N_y \times N_z}$ and $N_x, N_y, N_z$ are the grid dimensions. Performance is assessed via (i) map correlation (\textit{cc}) computed using \texttt{phenix.validation\_cryoem} \cite{afonine2018new}, and (ii) physical plausibility quantified by the MolProbity score using \texttt{phenix.molprobity} \cite{williams2018molprobity}. 

\textbf{Overcoming Local Optima.} We first illustrate the qualitative difference between methods on a representative challenging target, 8H1I, where the prior generates a conformation with an incorrect domain orientation (\textit{cc}$=0.42$) relative to the reference map (\Cref{fig:map_reward}a). As the target lies far from the prior's dominant mode, DPS fails to correct this topology and stagnates at a local optimum (\textit{cc}$=0.58$) even under the best hyperparameters tested. EmbedOpt, in contrast, successfully resolves the global domain rearrangement (\textit{cc}$=0.93$). See \Cref{app:fig:cryoem_examples} for additional examples.

\textbf{Improved Performance in Low-Density Prior Regimes.} 
To systematically quantify this advantage, we compare best-achieved performance (after sweeping learning rates) across the full dataset (\Cref{fig:map_reward}b). We define task ``difficulty" as $1 - $ prior map correlation, a proxy for the misalignment between the prior's dominant mode and the experimental target. Both methods perform comparably on ``easy" targets (difficulty $<0.4$), but EmbedOpt consistently outperforms DPS on ``hard" targets. 

\textbf{Engineering Robustness and Manifold Stability.} Beyond peak performance, EmbedOpt demonstrates superior engineering robustness (\Cref{fig:map_reward}c). 
DPS requires delicate hyperparameter tuning to find a narrow efficacy window of learning rates, as low learning rates provide insufficient guidance while high values produce unphysical structures (\Cref{app:fig:cryoem_failure_case}).  
EmbedOpt exhibits a stable performance plateau across learning rates spanning from $0.01$ to $1.0$, achieving higher average performance than DPS while substantially reducing the cost of hyperparameter search. This stability extends to stereochemical quality, where EmbedOpt maintains valid geometries even at large learning rates where DPS fails entirely (\Cref{app:fig:cryoem_molprobity_score}). \emph{Takeaway: On challenging targets requiring large conformational shifts, EmbedOpt outperforms DPS, robust across a $100\times$ learning-rate range.}

\subsection{Cryo-EM Map Fitting (Real Experimental Targets)}\label{subsec:exp:real-cryoem}

\begin{figure}[!t]
    \centering
    \includegraphics[width=1.0\linewidth]{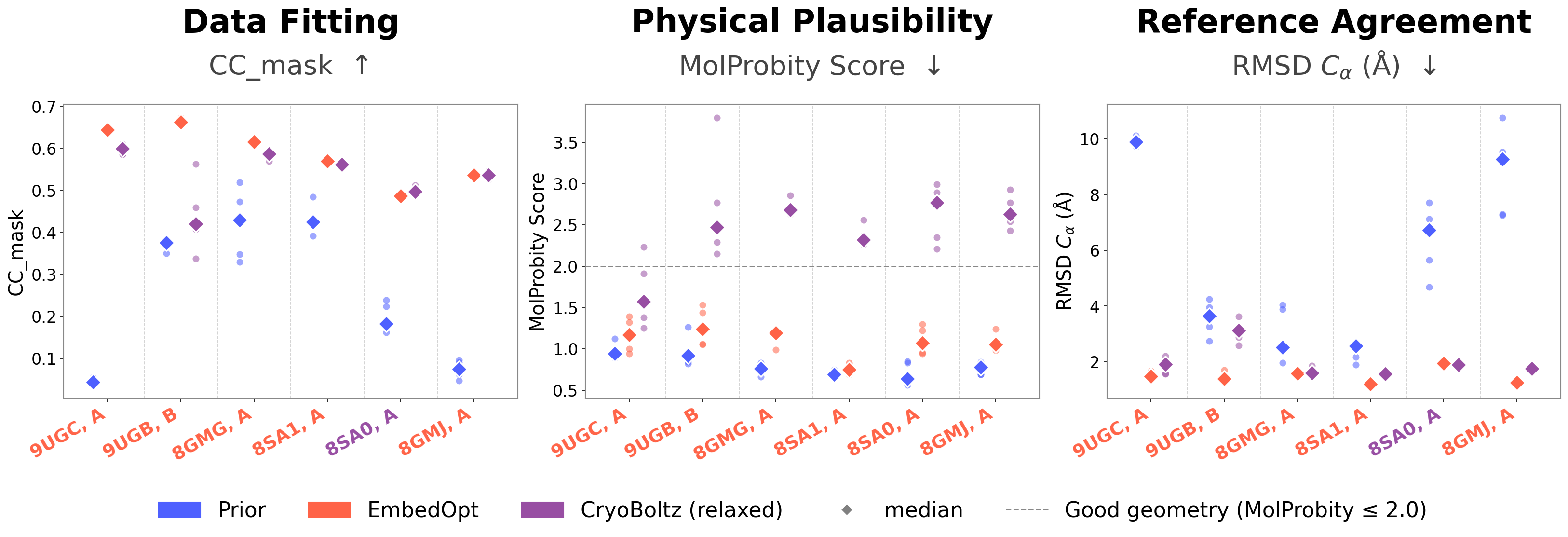}
    \caption{\textbf{EmbedOpt is competitive with or outperforms the CryoBoltz baseline on real experimental cryo-EM targets.} Performance of the unguided Prior (blue), EmbedOpt (red), and CryoBoltz (purple, after energy relaxation) across six experimental targets. Three aspects of model quality are assessed: Data Fitting (left, CC mask $\uparrow$), Physical Plausibility (center, MolProbity score $\downarrow$), and Reference Agreement (right, RMSD C$\alpha$ $\downarrow$).  Each of the 5 seed predictions are dots, with medians marked by filled diamonds. The x-axis label is colored by the method achieving the better median. EmbedOpt achieves a better median than CryoBoltz on all three metrics for 5/6 systems.}
    \label{fig:real_map_cryoboltz}
\end{figure}

To evaluate whether the robustness observed in synthetic settings transfers to real, noisy experimental data, we benchmark EmbedOpt on 6 real cryo-EM targets of varying  resolutions used to evaluate  CryoBoltz, a recent DPS-based method built on  Boltz-1. We  structure this evaluation as \textbf{two complementary experiments answering distinct questions}: (i) \emph{Does the embedding-vs-coordinate algorithmic choice still matter on real data?} — addressed by a same-model comparison of EmbedOpt vs.\ DPS within the Protenix prior (\Cref{app:fig:cryoboltz_comparison_si}); and (ii) \emph{Is EmbedOpt competitive with the published state-of-the-art?} — addressed by the cross-model comparison against CryoBoltz (\Cref{fig:real_map_cryoboltz}).

\textbf{Same-model Comparison: EmbedOpt vs.\ DPS in Protenix.} We first verify that the algorithmic advantages observed on synthetic data persist on real maps. As shown in \Cref{app:fig:real_map_lr_sensitivity}, EmbedOpt's median CC mask remains above DPS across the full learning-rate sweep $\alpha \in [0.0025, 0.5]$; DPS only matches EmbedOpt within a narrow window around $\alpha \sim 0.1$ before degrading sharply at higher rates. At a matched $\alpha = 0.1$ (with 100 diffusion steps), EmbedOpt still outperforms DPS on CC mask on all 6 systems while both methods maintain good stereochemical quality (MolProbity $\lesssim 2$, \Cref{app:fig:cryoboltz_comparison_si}). This mirrors the synthetic-benchmark findings. 

\textbf{Cross-model Comparison: EmbedOpt vs.\ CryoBoltz.} We then compare against the published baseline  \textbf{CryoBoltz}. We note that this comparison conflates the base model choice (Protenix vs.\ Boltz-1) with inference-time schemes, while the previous experiment disentangles these factors. 
For EmbedOpt, we use 100 diffusion steps and the same learning rate $\alpha=0.1$ across all six systems (9UGC, 9UGB, 8GMG, 8SA1, 8SA0, 8GMJ); for CryoBoltz, we use its default published settings (in particular, with 200 diffusion steps). To ensure a fair comparison, both methods receive identical inputs, undergo the same post-prediction energy relaxation, and are evaluated with the same Phenix-based metric pipeline; full protocol details are provided in \Cref{app:par:cryoboltz_protocol}. We evaluate three metrics (\Cref{fig:real_map_cryoboltz}): data fitting (CC mask), 
physical plausibility (MolProbity score) 
and reference agreement (RMSD C$\alpha$ relative to the deposited structure). EmbedOpt achieves  better results (median) than  CryoBoltz across all three metrics on 5 out of 6 systems; on the remaining system (8SA0), CryoBoltz achieves a marginally higher CC mask at the cost of a substantially worse MolProbity score. Per-system structural visualizations of the median-CC predictions for all methods (Prior, DPS, EmbedOpt, CryoBoltz) on each of the six targets are provided in \Cref{app:fig:real_cryoem_gallery_A,app:fig:real_cryoem_gallery_B}.

In addition, coordinate-based steering is \textbf{substantially more susceptible} to reward-fitting artifacts than embedding-space optimization. Before energy relaxation, CryoBoltz predictions achieve high map correlation but exhibit severely degraded geometry (MolProbity $\gg 2.0$); the energy relaxation improves its geometry but remains suboptimal compared to EmbedOpt, as severe stereochemical violations cannot be fully repaired by local optimization (\Cref{app:fig:cryoboltz_comparison_si}). EmbedOpt also degrades at extreme learning rates (\Cref{fig:map_reward}c, \Cref{app:fig:cryoem_molprobity_score}), so the difference is quantitative rather than qualitative; however, the degradation onset is  delayed and milder. 
\emph{Takeaway: Our within-model evidence shows the algorithmic mechanism transfers to real data; our cross-model comparison shows the resulting pipeline is competitive with or outperforms the published state-of-the-art on 5 out of 6 targets while producing substantially better stereochemical geometry under matched post-processing.}

\section{Discussion}
In this work, we introduce EmbedOpt, an inference-time method that steers biomolecular sequence-to-structure diffusion models by optimizing conditional embeddings.  
We show that EmbedOpt matches or outperforms DPS, with  greater hyperparameter robustness and optimization stability  across multiple structure-determination inverse problems. EmbedOpt provides a steering axis orthogonal to coordinate-based DPS, suggesting a natural future direction of integrating the two, e.g. by interleaving EmbedOpt's prior updates with DPS-style posterior sampling in an Empirical-Bayes-like fashion. 

\paragraph{Limitations and Future Works.}
EmbedOpt is not without failure cases, and characterizing when it does \emph{not} work is informative. Its robustness is quantitative rather than absolute, and we identify three algorithmic regimes that bound its effectiveness; \Cref{app:sec:failure_modes} provides extended per-system discussion, proposed remedies, and a consolidated discussion of evaluation-scope caveats.

\textit{Manifold coverage.} EmbedOpt only navigates the structural manifold induced by the pretrained model: when the target conformation is poorly represented, neither EmbedOpt nor DPS recovers it. This regime arises both when evolutionary information is limited (e.g.\ no MSA input; \Cref{app:sec:ablation_msa}) and on out-of-distribution targets such as 8K23 in our cryo-EM benchmark (\Cref{app:fig:cryoem_examples}). Closing this gap requires expanded pretraining or model fine-tuning rather than inference-time methods alone.

\textit{Manifold deviation.} At aggressive learning rates or with reward signals far from the training distribution, EmbedOpt can deviate from the structural manifold, producing stereochemical degradation that energy relaxation cannot fully repair (\Cref{fig:map_reward}c) — though this boundary is farther than for DPS. Potential solutions include physics-aware regularization during optimization and, more fundamentally, pretraining objectives that explicitly regularize the embedding space (e.g., via variational bottlenecks).

\textit{Local optima.} While EmbedOpt typically navigates around the coordinate-space local optima where DPS stagnates (\Cref{fig:map_reward}a, e.g.\ 8H1I), the trust-region guarantee (\Cref{prop:cross-time-minimal}) is itself local: when the embedding-space gradient vanishes at a non-global optimum, EmbedOpt can be stuck. We observe this inverse failure on a small subset of targets (e.g., 8F2R, \Cref{app:fig:cryoem_examples}), where DPS's noisy coordinate updates escape an embedding-space local mode that EmbedOpt does not. The two methods thus have complementary failure modes in different parameter spaces; hybrid schemes combining EmbedOpt's stable optimization with coordinate-space stochasticity are a promising direction.

\begin{ack}
The authors thank Alisia Fadini, Colin Kalicki, and Tetiana Parshakova for valuable discussions.
\end{ack}

\bibliography{reference}
\bibliographystyle{abbrvnat}

\newpage
\appendix

\crefalias{section}{appendix}
\crefalias{subsection}{appendix}

\crefname{section}{appendix}{appendices}
\Crefname{section}{Appendix}{Appendices}
\crefname{subsection}{appendix}{appendices}
\Crefname{subsection}{Appendix}{Appendices}

\section{Method Details}\label{app:sec:algorithm}

\subsection{EmbedOpt Algorithm Adapted to AlphaFold~3 Sampling Scheme}

\begin{algorithm}[t]
\caption{Embed-Opt Adapted for AlphaFold~3 Sampling Scheme}
\label{alg:embed-opt-stochastic}
\begin{algorithmic}[1] 
\STATE  \textbf{Input:} differentiable reward function $R$,  original conditional embedding $c=\{z,s\}$, denoiser network $\hat x_\theta$, base reward learning rate $\alpha$, total steps $T$,  noise  schedule $\{\sigma_t\}_{t=0}^T$, minimum noise amplification level $\gamma_0=1.0$,  noise amplification factor $\gamma=0.8$, noise scale $\rho=1.003$, step scale $\eta=1.5$ 
\STATE  \textbf{Output:} final sample $x_0$

\STATE Initialize embedding $c_T \gets c$
\STATE Sample initial coordinate $x_T \sim \mathcal{N}(0,\sigma_T^2)$
\FOR{$t = T, T-1, \dots, 1$}
 \IF{$\sigma_{t-1} > \gamma_{\min}$}
    \STATE Add noise to coordinate  $x_t \gets x_t + \rho \sqrt{(\gamma+1)^2 - 1}\,\sigma_t\,\epsilon_t,
      \quad \epsilon_t \sim \mathcal{N}(0,\mathbb{I})$ 
    \STATE Amplify noise level $\sigma_t \gets (\gamma + 1)\sigma_t$
  \ENDIF
  \STATE\linelabel{alg:line:denoising-prediction:alphafold}
   Make denoised prediction $\hat{x}_0 \gets \hat{x}_\theta(x_t, \colorEmbedopt{c_t=(s_t,z_t)} ,  \sigma_t)$ 
    \STATE Compute gradient $g_{c_t} \gets [g_{s_t}, g_{z_t} ]$ where $g_{s_t} \gets \nabla_{s_t} R(\hat x_0)$ and $g_{z_t} \gets \nabla_{z_t} R(\hat x_0)$
    \STATE Normalize gradient by RMS $\bar g_{c_t} \gets [\bar g_{s_t}, \bar g_{z_t}]$ where $\bar g_{s_t} \gets g_{s_t}/\sqrt{\frac{1}{d_s} \sum_{i=1}^{d_s} \left(g_{s_t}^{(i)}\right)^2} $ and $\bar g_{z_t}
\gets g_{z_t} / \sqrt{\frac{1}{d_z} \sum_{i=1}^{d_z} \left(g_{z_t}{(i)}\right)^2}$ 

  \STATE Update embedding   $\colorEmbedopt{c_{t-1} \gets c_t + \alpha \bar g_{c_t}}$ 
  \STATE  \linelabel{alg:line:coordupdate:alphafold} Update  coordinate $x_{t-1} \gets x_t + \eta_t \big[\hat x_\theta(x_t, \colorEmbedopt{c_{t-1}}, \sigma_{t-1})-x_t\big]$ where    $\eta_t =\frac{ \sigma_t-\sigma_{t-1}}{\sigma_t} * \eta$
\ENDFOR

\end{algorithmic}
\end{algorithm}

The original sampler in AlphaFold~3 adopts the stochastic sampling scheme proposed by \citet{karras2022elucidating}, which injects a small amount of noise into the coordinates $x_t$ at the initial steps (Lines 6–9) until a minimum noise level is reached. It further employs a step scale of $\eta = 1.5$ to increase the diffusion step size (Line 14). Protenix, as an open-source reproduction of AlphaFold~3, uses the same settings.

We present the full EmbedOpt algorithm, compatible with AlphaFold~3–style  sampling scheme, in \Cref{alg:embed-opt-stochastic}. We adopt the same default hyperparameters in all protein benchmark experiments.

Finally, note that when setting $\gamma=0$ and step scale $\eta=1$ in 
\Cref{alg:embed-opt-stochastic}, it recovers the deterministic sampler in \Cref{alg:embed-opt} with a standard diffusion step size.

\subsection{Proposition on the Monotone Surrogate Reward Improvement of EmbedOpt}\label{app:subsec:prop}
We provide a formal statement and proof of \Cref{prop:cross-time-minimal-informal}. 

\paragraph{Setup.}
Define the  surrogate reward at $(x_t, c_t, \sigma_t)$, 
\begin{align}
F(x_t,c_t, \sigma_t)\;\coloneqq\;R \big(\hat x_\theta(x_t,c_t,\sigma_t)\big). 
\end{align}
EmbedOpt's update at step $t$ is 
\begin{align}
c_{t-1} &= c_t + \alpha_t \nabla_{c_t} F(x_t,c_t, \sigma_t) \label{eq:emebdopt-update-c} \\
x_{t-1} &= x_t + \sigma_t^{-1} \Delta \sigma_t \Big(\hat x_\theta(x_t,c_{t-1},\sigma_t)-x_t\Big), \label{eq:embedopt-update-x}
\end{align}
where $\Delta\sigma_t\coloneqq \sigma_t-\sigma_{t-1}>0$ is the noise level step size. 

\begin{assumption}[Local $c$-smoothness of $F$]\label{ass:c-smooth-local}
Given $(x_t, c_t, \sigma_t)$, there exists a neighborhood $\mathcal{C}$ of $c_t$ with a radius $r_c$ such that $\forall c, c' \in \mathcal{C}$, the function $c\mapsto F(x_t,c, \sigma_t)$ is $L_c$-smooth:
\begin{align}
F(x_t,c', \sigma_t)\ge F(x_t,c, \sigma_t)+\nabla_{c} F(x_t,c, \sigma_t)^\top(c'-c)-\frac{L_c}{2}\|c'-c\|^2.
\end{align}
where $\mathcal{C}, r_c>0, L_c>0$ depend on $(x_t, c_t, \sigma_t)$. 
\end{assumption}

\begin{assumption}[Local $c$-boundedness of $\hat x_\theta$]\label{ass:c-boundedness-local}
    For the given $(x_t, \mathcal{C}, \sigma_t)$ in \Cref{ass:c-smooth-local}, $ \max_{c \in \mathcal{C}} \| \hat x_\theta (x_t, c, \sigma_t)\| \leq \|\hat x_\theta\|_{\infty, \mathcal{C}}$ for some constant $\|\hat x_\theta\|_{\infty, \mathcal{C}}$ depending on $(x_t, \mathcal{C}, \sigma_t)$.
\end{assumption}

\begin{assumption}[Local $x$-Lipschitz continuity of $F$]\label{ass:x-lip-local} 
For the given $(x_t, \mathcal{C}, \sigma_t)$ in \Cref{ass:c-smooth-local}, there exists a neighborhood $\mathcal{X}$ of $x_t$ with a radius $r_x$ such that $\forall x, x' \in \mathcal{X}, \forall c \in \mathcal{C}$, the map $x\mapsto F(x,c, \sigma_t)$ is $G_x$-Lipschitz:
\begin{align}
|F(x',c, \sigma_t)-F(x,c, \sigma_t)|\le G_x\|x'-x\|, 
\end{align}
where $\mathcal{X},  r_x>0, G_x>0$ depend on $(x_t, \mathcal{C},\sigma_t)$. 
\end{assumption}

\begin{assumption}[Local $\sigma$-Lipschitz continuity]\label{ass:sigma-uniform-local}
For the given $(\mathcal{X}, \mathcal{C}, \sigma_t)$ in \Cref{ass:c-smooth-local,ass:x-lip-local}, 
there exists a maximum step size  $r_\sigma>0$ such that $\forall x \in \mathcal{X}, \forall c \in \mathcal{C}, \forall \sigma$ such that $0\leq \sigma_t -\sigma \leq r_\sigma$,  the map $\sigma \mapsto F(x, c, \sigma)$  is $G_\sigma$-Lipschitz: 
\begin{align}
|F(x,c, \sigma)-F(x,c, \sigma_t)| \le G_\sigma (\sigma_t-\sigma)
\end{align}
where $r_\sigma>0, G_\sigma>0$ depend on $(\mathcal{X}, \mathcal{C}, \sigma_t)$. 
\end{assumption}

\begin{proposition}[EmbedOpt one-step surrogate reward improvement]\label{prop:cross-time-minimal}
Given $(x_t, c_t, \sigma_t)$, 
under \Cref{ass:c-smooth-local,ass:x-lip-local,ass:sigma-uniform-local}, if the EmbedOpt learning rate $\alpha_t$ satisfies
\begin{align}
0<\alpha_t \le  \alpha_{max}\coloneqq \min \{ \frac{1}{L_c}, \frac{r_c}{\| \nabla_{c_t} F(x_t, c_t, \sigma_t) \|} \} \label{eq:trust-region-learning-rate}
\end{align}
and the noise level step size  $\Delta \sigma_t = \sigma_t - \sigma_{t-1}$ satisfies  
\begin{align}
\begin{split}
\Delta\sigma_t \leq  \Delta_{\max} \coloneqq  \min \{& \left[ \sigma_t^{-1}G_x (\|\hat x_\theta\|_{\infty, \mathcal{C}} + \|x_t \|) \;+\;  G_{\sigma_t}
\;\right]^{-1}
\frac{\alpha_t}{2}\,\big\|\nabla_{c_t} F(x_t,c_t, \sigma_t)\big\|^2, \\
&  \frac{\sigma_t r_x}{\|\hat x_\theta\|_{\infty, \mathcal{C}} + \|x_t \|}, r_\sigma \}
\end{split} 
\label{eq:trust-region-step-size}
\end{align}
then the surrogate reward under the EmbedOpt update \Cref{eq:emebdopt-update-c,eq:embedopt-update-x} is non-decreasing across time:
\begin{align}
F(x_{t-1},c_{t-1},\sigma_{t-1}) & \geq F(x_t,c_t,\sigma_t). \label{eq:reward-monotonicity}
\end{align}
Moreover, the quantitative bound holds 
\begin{align}
&F(x_{t-1},c_{t-1}, \sigma_{t-1}) \\
& \ge
F(x_t,c_t, \sigma_t)
+
\frac{\alpha_t}{2}\big\|\nabla_{c_t} F(x_t,c_t, \sigma_t)\big\|^2
-
G_x\frac{\Delta\sigma_t}{\sigma_t}\|\hat x_\theta(x_t, c_{t-1}, \sigma_t)-x_t\|
-
G_\sigma \Delta\sigma_t.
\label{eq:gap-min}
\end{align}
\end{proposition}

\paragraph{Remarks.}
\begin{itemize}
    \item \Cref{ass:c-smooth-local,ass:c-boundedness-local,ass:x-lip-local,ass:sigma-uniform-local} are on the smoothness and boundedness of the surrogate reward $F(\cdot, \cdot, \cdot)$ in a neighborhood of $(x_t, c_t, \sigma_t)$,  which  hold for sufficiently regular denoiser network $\hat x_\theta(\cdot, \cdot, \cdot)$ and the reward function $R(\cdot)$.  
    \item The learning rate requirement for $\alpha_t$ in \Cref{eq:trust-region-learning-rate} is a classical gradient ascent condition, satisfied for sufficiently small $\alpha_t$. This assumption ensures that the embedding update yields non-negative reward improvement within a trust region. 
    \item The assumption on the noise level step size $\Delta \sigma_t$ in \Cref{eq:trust-region-step-size} is another trust-region assumption which ensures the reward gain by updating embedding  dominates the worst-case surrogate reward decrease due to $x_t \to x_{t-1}$ and $\sigma_t \to \sigma_{t-1}$. Noticeably, when the embedding gradient $\nabla_{c_t} F(x_t, c_t, \sigma_t) \neq 0$, we can pick sufficiently small noise level step size $\Delta \sigma_t$ for this assumption to hold. However, if $c_t$ is a local optimum with $\nabla_{c_t} F(x_t, c_t, \sigma_t)=0$, this assumption is vacuous and 
    there is no reward ascent guarantee against the shift caused by coordinate and noise level updates.
\end{itemize}

\begin{proof}
We first evaluate the surrogate reward gain from the embedding update $c_{t} \to c_{t-1}$ in Step 1, for a sufficiently small  learning rate $\alpha_t$ in \Cref{eq:trust-region-learning-rate}. 
Then we compute the reward lower bound while updating $x_t\to x_{t-1}$ in Step 2 and $\sigma_t \to \sigma_{t-1}$ in Step 3. Finally, under the step-size assumption in \Cref{eq:trust-region-step-size}, we prove the non-decreasing surrogate reward given a  full update from $(x_t, c_t, \sigma_t)$ to $(x_{t-1}, c_{t-1}, \sigma_{t-1})$ in Step 4, which leads to \Cref{eq:reward-monotonicity,eq:gap-min}. 

\paragraph{Step 1: gain from embedding update $c_t \to c_{t-1}$ given fixed $x_t$ and $\sigma_t$.}
When the learning rate  $\alpha_t \leq \frac{r_c}{\|\nabla_{c_t} F(x_t, c_t, \sigma_t)\|}$ (implied by \Cref{eq:trust-region-learning-rate}), we have 
$ \|c_{t-1} -c_t\|   =  \alpha_t \| \nabla_{c_t}  F(x_t, c_t, \sigma_t)\| \leq r_c$, 
and hence $c_{t-1} \in \mathcal{C}$ (recall $\mathcal{C}$ is defined as the neighborhood of $c_t$ with a radius $r_c$). 

By \Cref{ass:c-smooth-local} with $(c, c')=(c_t,c_{t-1}) \in \mathcal{C}$, 
\begin{align}
F(x_t,c_{t-1}, \sigma_t)
&\geq F(x_t,c_t, \sigma_t)
+ \nabla_{c_t} F(x_t,c_t, \sigma_t)^\top(c_{t-1}-c_t)
- \frac{L_c}{2}\|c_{t-1}-c_t\|^2.
\end{align}
Using $c_{t-1}-c_t=\alpha_t \nabla_{c_t} F(x_t,c_t, \sigma_t)$,
\begin{align}
F(x_t,c_{t-1}, \sigma_t)
&\ge F(x_t,c_t, \sigma_t)
+ \alpha_t\Big(1-\frac{L_c\alpha_t}{2}\Big)\|\nabla_{c_t} F(x_t,c_t, \sigma_t)\|^2.
\end{align}
If $0<\alpha_t\le 1/L_c$, then $1-\frac{L_c\alpha_t}{2}\ge \tfrac12$, hence
\begin{align}
F(x_t,c_{t-1}, \sigma_t)
\geq F(x_t,c_t, \sigma_t) + \frac{\alpha_t}{2}\|\nabla_{c_t} F(x_t,c_t, \sigma_t)\|^2.
\label{eq:step1-min}
\end{align}

\paragraph{Step 2: reward lower bound from coordinate update $x_t \to x_{t-1}$ given fixed $c_{t-1}$ and $\sigma_t$.}
Note that 
\begin{align}
    \|x_{t-1} - x_t \|&= \Delta \sigma_t \sigma_t^{-1} \|\hat x_\theta (x_t, c_{t-1}, \sigma_t) - x_t \| \\
    &\leq \Delta \sigma_t \sigma_t^{-1} \big (\|\hat x_\theta (x_t, c_{t-1}, \sigma_t)\| +\| x_t \|  \big) \\
    & \leq  \Delta \sigma_t \sigma_t^{-1} \big ( \|\hat x_\theta \|_{\infty, \mathcal{C}} +\| x_t \|  \big)
\end{align}
where the last line follows from \Cref{ass:c-boundedness-local}. 

When the noise-level step size $\Delta \sigma_t \leq \frac{\sigma_t r_x}{ \|\hat x_\theta \|_{\infty, \mathcal{C}} +  \|x_t\|}$ implied by \Cref{eq:trust-region-step-size}, we have that $\|x_{t-1} - x_t\| \leq r_x$ and hence $x_{t-1} \in \mathcal{X}$ (recall $\mathcal{X}$ is defined as the neighborhood of $x_t$ with a radius $r_x$). 

By \Cref{ass:x-lip-local} with $(x, x')=(x_t,x_{t-1}) \in \mathcal{X}$ ,
\begin{align}
F(x_{t-1},c_{t-1}, \sigma_t)
\geq F(x_t,c_{t-1}, \sigma_t) - G_x \|x_{t-1}-x_t\|.
\end{align}
Using $x_{t-1} - x_t=  \sigma_t^{-1}  \Delta \sigma \| \hat x_\theta (x_t, c_{t-1}, \sigma_t) - x_t\|$, we have 

\begin{align}
F(x_{t-1},c_{t-1}, \sigma_t)
& \geq F(x_t,c_{t-1}, \sigma_t) -  \sigma_t^{-1}  G_x \Delta \sigma_t \| \hat x_\theta (x_t, c_{t-1}, \sigma_t) - x_t\|.
\label{eq:step2-min}
\end{align}

\paragraph{Step 3: reward lower bound from noise level update $\sigma_t \to \sigma_{t-1}$ given fixed $x_{t-1}$ and $c_{t-1}$.}
By \Cref{ass:sigma-uniform-local} with $(\sigma, \sigma') = (\sigma_t, \sigma_{t-1})$ (satisfied when $\Delta \sigma_t \leq r_\sigma$ by \Cref{eq:trust-region-step-size}), 
\begin{align}
F(x_{t-1},c_{t-1}, \sigma_{t-1})
\geq F(x_{t-1},c_{t-1}, \sigma_t) - G_\sigma \Delta\sigma_t.
\label{eq:step3-min}
\end{align}

\paragraph{Step 4: full update}
Combining \Cref{eq:step1-min,eq:step2-min,eq:step3-min}, 
\begin{align}
    &F(x_{t-1}, c_{t-1}, \sigma_{t-1}) \\
   & \geq F(x_{t-1},c_{t-1}, \sigma_t) - G_\sigma \Delta\sigma_t \\ 
    &\geq F(x_t,c_{t-1}, \sigma_t) -  \sigma_t^{-1}  G_x \Delta \sigma_t \| \hat x_\theta (x_t, c_{t-1}, \sigma_t) - x_t\| - G_\sigma \Delta \sigma_t \\ 
    &\geq F(x_t,c_t, \sigma_t) + \frac{\alpha_t}{2}\|\nabla_{c_t} F(x_t,c_t, \sigma_t)\|^2 -  \sigma_t^{-1}  G_x \Delta \sigma_t \| \hat x_\theta (x_t, c_{t-1}, \sigma_t) - x_t\| - G_\sigma \Delta \sigma_t \label{eq:embedopt-bound-deriv-lastline}
\end{align}
The last line recovers the bound in \Cref{eq:gap-min}.  

When the noise level also satisfies the following trust region as part of the condition in \Cref{eq:trust-region-step-size}:
\begin{align} 
   \Delta\sigma_t & \leq   \left[ \sigma_t^{-1}G_x (\|\hat x_\theta\|_{\infty, \mathcal{C}} + \|x_t \|) \;+\;  G_{\sigma_t}
\;\right]^{-1}
\frac{\alpha_t}{2}\,\big\|\nabla_{c_t} F(x_t,c_t, \sigma_t)\big\|^2, 
\end{align}
which implies 
\begin{align} 
    \frac{\alpha_t}{2}\|\nabla_{c_t} F(x_t,c_t, \sigma_t)\|^2 -  \sigma_t^{-1}  G_x \Delta \sigma_t \| \hat x_\theta (x_t, c_{t-1}, \sigma_t) - x_t\| - G_\sigma \Delta \sigma_t  \geq 0,
\end{align}
 we have the non-decreasing surrogate reward for a full EmbedOpt update across time: 
\begin{align}
    F(x_{t-1}, c_{t-1}, \sigma_{t-1}) \geq F(x_t, c_t, \sigma_t). 
\end{align}
\end{proof}

\subsection{Derivation of the First-order Taylor Approximation to the  EmbedOpt Update in \Cref{eq:embedopt-update-taylor}}\label{app:subsec:taylor}

We apply a first-order Taylor approximation of $\hat x_\theta(x_t, c_{t-1}, \sigma_t)$ around $c_t$: 
\begin{align}
\hat x_\theta (x_t, c_{t-1}, \sigma_t) &\approx \hat x_\theta (x_t, c_{t}, \sigma_t) + J_{c_t} (c_{t-1} - c_t) 
\end{align}
where the remainder  is $\smallo{(\|c_{t-1} - c_t\|)}$, and $J_{c_t} \coloneqq \nabla_{c_t} \hat x_\theta(x_t, c_t, \sigma_t)$. 

Using the embedding update rule, i.e. $c_{t-1} - c_t = \alpha_t \nabla_{c_t} R(\hat x_\theta(x_t, c_t, \sigma_t))$, 
\begin{align}
    \hat x_\theta (x_t, c_{t-1}, \sigma_t)&\approx \hat x_\theta (x_t, c_{t}, \sigma_t) + \alpha_t J_{c_t}^\top \nabla_{c_t}R(\hat x_\theta (x_t, c_t, \sigma_t))  
\end{align}

Applying  chain rule to $\nabla_{c_t}R(\hat x_\theta (x_t, c_t, \sigma_t))$ and write $\hat x_0=\hat x_\theta (x_t, c_t, \sigma_t)$,  
\begin{align}\label{eq:embedopt-xhat-update}
      \hat x_\theta (x_t, c_{t-1})&\approx \hat x_\theta (x_t, c_{t}, \sigma_t) + \alpha_t J_{c_t} J_{c_t}^\top \nabla_{\hat x_0}R(\hat x_0). 
\end{align}

Finally, EmbedOpt's full one-step update of $x_t \to x_{t-1}$ can be approximated by 
\begin{align}
    x_{t-1} &= x_t + \eta_t \left[ \hat x_\theta (x_t, c_{t-1}, \sigma_t) -x_t\right]
    \\ 
    &\approx x_t + \eta_t \left[\hat x_\theta (x_t, c_{t}, \sigma_t) + \alpha_t J_{c_t} J_{c_t}^\top \nabla_{\hat x_0}R(\hat x_0) -x_t \right],
\end{align}
which recovers \Cref{eq:embedopt-update-taylor}.

\section{Diffusion Posterior Sampling and Applications to Biophysical Inverse Problems}\label{app:sec:dps}
\subsection{DPS Background}\label{app:subsec:dps-background}

Diffusion posterior sampling \citep[DPS,][]{chung2022diffusion} is a popular approach for solving inverse problems by leveraging diffusion models as flexible priors in a Bayesian posterior sampling framework. Consider a likelihood  of the form  $p^{R}(x_0) \propto \exp \{  R(x_0) \}$ where $R(\cdot)$ is a reward function that encodes alignments with experimental measurement. Given a diffusion model prior $p(x_0 \mid c)$, DPS aims to (approximately) sample from the resulting posterior which is proportional to $p(x_0 \mid c) \exp \{  R(x_0)\}$.

To sample from the posterior, the probability flow ODE in \Cref{eq:prob-flow-ode} is modified to include the additional gradient term of the log expected exponentiated reward, which is referred to as \textit{likelihood guidance}, 
\begin{align}
      \dd x_t &=  -\dot \sigma(t) \sigma(t) \Big[ \nabla_{x_t} \log p_t(x_t)  \colorDPS{ + \nabla_{x_t}\log \mathbb{E}\left[\exp \{   R(x_0)\}\mid x_t  \right] }\Big] \dd t
\end{align}

While intractable, DPS approximates the log  expected exponentiated reward by the reward evaluated at the posterior expectation

\begin{align}
 \log \mathbb{E}\left[\exp \{   R(x_0)\}\mid x_t \right]  & \approx  \log \exp \{  R\left( \mathbb{E} \left[x_0 \mid x_t \right]\right) \} =   R\left(\mathbb{E}[x_0 \mid x_t]\right). 
\end{align}

Using the denoising prediction $\hat x_\theta (x_t,c, \sigma_t) \approx \mathbb{E}[x_0 \mid x_t]$, DPS's discrete-time update is given by

\begin{align}\label{eq:dps-update-original}
    x_{t-1} &= x_t + (\sigma_t - \sigma_{t-1}) \sigma_t \left[ \frac{\hat x_\theta (x_t, c, \sigma_t) - x_t}{\sigma_t^2} \colorDPS{ + \nabla_{x_t} R ( \hat x_\theta(x_t, c, \sigma_t))} \right] \\
    &= x_t + \underbrace{\frac{\sigma_t - \sigma_{t-1}}{\sigma_t}}_{\coloneqq \eta_t} \left[ \hat x_\theta (x_t, c, \sigma_t) - x_t \colorDPS{ + \sigma_t^2 \nabla_{x_t} R ( \hat x_\theta(x_t, c, \sigma_t) )} \right]. 
\end{align}

In practice,   using an adaptive ``learning rate" schedule $\alpha_t$  that controls the likelihood guidance strength can be helpful, 
\begin{align}\label{eq:dps-update}
    x_{t-1} &= x_t + \eta_t (\hat x_\theta (x_t, \sigma_t) - x_t \colorDPS{ + \alpha_t \nabla_{x_t} R ( \hat x_\theta(x_t, \sigma_t) })). 
\end{align}
The optimal learning rate schedule would depend on specific downstream tasks. 

\paragraph{Connection to Reweighting Likelihood.} Setting $\alpha_t = \sigma_t^2 w$ admits a Bayesian interpretation in which the likelihood is reweighted by a factor of $w$ relatively to the prior, with the targeted posterior $\propto p(x_0) \exp \{w R(x_0)\}$. Since DPS is only an approximate inference method, this interpretation should be viewed as heuristic rather than exact.

\begin{algorithm}[t]
\caption{Diffusion Posterior Sampling \citep[DPS,][]{chung2022diffusion}}
\label{alg:dps}
\begin{algorithmic}[1] 
\STATE  \textbf{Input:} differentiable reward function $R$, original conditional embedding $c$, denoiser network $\hat x_\theta$,  adaptive learning rate $\alpha_t$, total steps $T$, noise  schedule $\{\sigma_t\}_{t=0}^T$\\
\STATE  \textbf{Output:} final sample $x_0$

\STATE Sample initial coordinate $x_T \sim \mathcal{N}(0,\sigma_T^2)$

\FOR{$t = T, T-1, \dots, 1$}
  \STATE 
   Make denoised prediction $\hat{x}_0 \gets \hat{x}_\theta(x_t, c \color{black}, \sigma_t)$ 
  \STATE  Update  coordinate $x_{t-1} \gets x_t + \eta_t \big[\hat x_0 -x_t  \colorDPS{ + \alpha_t \nabla_{x_t} R(\hat x_0 )} \big]$ where    $\eta_t =\frac{ \sigma_t-\sigma_{t-1}}{\sigma_{t}}$
\ENDFOR

\end{algorithmic}
\end{algorithm}

\paragraph{Generic DPS Algorithm.} We summarize the DPS algorithm with a generic learning rate schedule $\alpha_t$ and under a standard ODE sampling scheme in \Cref{alg:dps}.

\subsection{DPS Algorithm Adapted to AlphaFold~3 Sampling Scheme}
\begin{algorithm}[!t]
\caption{DPS Adapted for AlphaFold~3 Sampling Scheme}
\label{alg:dps-alphafold}
\begin{algorithmic}[1] 
\STATE  \textbf{Input:} differentiable reward function $R$,  original conditional embedding $c=\{z,s\}$, denoiser network $\hat x_\theta$, base reward learning rate $\alpha$, total steps $T$,  noise  schedule $\{\sigma_t\}_{t=0}^T$, minimum noise amplification level $\gamma_0=1.0$,  noise amplification factor $\gamma=0.8$, noise scale $\rho=1.003$, step scale $\eta=1.5$ 
\STATE  \textbf{Output:} final sample $x_0$

\STATE Sample initial coordinate $x_T \sim \mathcal{N}(0,\sigma_T^2)$
\FOR{$t = T, T-1, \dots, 1$}
 \IF{$\sigma_{t-1} > \gamma_{\min}$}
    \STATE Add noise to coordinate  $x_t \gets x_t + \rho \sqrt{(\gamma+1)^2 - 1}\,\sigma_t\,\epsilon_t,
      \quad \epsilon_t \sim \mathcal{N}(0,\mathbb{I})$ 
    \STATE Amplify noise level $\sigma_t \gets (\gamma + 1)\sigma_t$
  \ENDIF
  \STATE 
   Make denoised prediction $\hat{x}_0 \gets \hat{x}_\theta(x_t, c,  \sigma_t)$ 
    \STATE Compute gradient $g_{x_t} \gets \nabla_{x_t} R(x_0)$
    \STATE Normalize gradient $\bar g_{x_t} \gets \|\hat x_0 - x_t \| \frac{g_{x_t}}{\|g_{x_t}\|}$
  \STATE  Update  coordinate $x_{t-1} \gets x_t + \eta_t \big[\hat x_\theta(x_t, c, \sigma_t) -x_t \colorDPS{+ \alpha \bar g_{x_t}}\big]$ where    $\eta_t =\frac{ \sigma_t-\sigma_{t-1}}{\sigma_t} * \eta$
\ENDFOR

\end{algorithmic}
\end{algorithm}

We follow the gradient normalization strategy in \citet{maddipatla2025inverse} (official implementation  can be found in \url{https://github.com/sai-advaith/guided_alphafold}, which is also built on the Protenix model), and summarize the DPS algorithm adapted for AlphaFold~3 sampling scheme in \Cref{alg:dps-alphafold}.

\section{Details on Synthetic Illustration in \Cref{fig:overview} (a)} \label{app:sec:synthetic}
In the synthetic example in \Cref{fig:overview} (a), we consider a diffusion prior model to be a 1-dimensional Gaussian $p(x_0 \mid c) = \mathcal{N}(x_0 \mid c, 0.5^2)$ conditioned on the location parameter $c=5$. We set the diffusion noise schedule $\sigma(t)=t$ for $t\in[0,1]$. Since $p(x_0 \mid c)$ is Gaussian, we can access the conditional expectation $\mathbb{E}[x_0 \mid x_t] \, \forall (x_t,t)$ without training a denoiser network. 

The measurement likelihood is given by $\mathcal{N}(y \mid x_0, 1)$ with the measurement $y=20$. This setting simulates the case where the prior model has low probability mass on the experimental measurements. 

We run the unguided prior model sampling, DPS, and EmbedOpt using $T=1,000$ uniform timesteps over $[0,1]$. 

To connect likelihood reweighting and the DPS's learning rate (as discussed in \Cref{app:subsec:dps-background}), we set $\alpha_t= \sigma_t^2 w$ in \Cref{alg:dps} where $w$ is a weighting parameter. More concretely,
\begin{itemize}
    \item When $w=1$, the exact posterior $p(y\mid x_0, c) \propto \mathcal{N}(x_0 \mid c,0.5^2)\mathcal{N}(y \mid x_0, 1)$, and the DPS update is 
    $x_{t-1} = x_t + \eta_t \left[\hat x_\theta(x_t, c, \sigma_t) - x_t + \sigma_t^2 \nabla_{x_t} \log N(y \mid \mathbb{E}[x_0 \mid x_t], 1)\right]$. 
    
    \item When $w=100$, the exact posterior is $p(y \mid x_0, c)\propto \mathcal{N}(x_0 \mid c, 0.5^2)\mathcal{N}(y \mid x_0,1)^w$, and the DPS update is 
     $x_{t-1} = x_t + \eta_t \left[\hat x_\theta(x_t, c, \sigma_t) - x_t + \sigma_t^2 100 \nabla_{x_t} \log N(y \mid \mathbb{E}[x_0 \mid x_t], 1)\right]$.

\end{itemize}

The EmbedOpt results in \Cref{fig:overview} (a) are obtained from  running \Cref{alg:embed-opt} with  $\alpha_t= \alpha \frac{1}{\|\nabla_{c_t} R(\hat x_0)\|}$ and $\alpha=0.1$. In practice we found a range of $\alpha$ from $0.05$ to $5$ to work well, that is, to be able to maximize the log likelihood.

\section{Additional Related Works}\label{app:sec:additional-related-works}

\paragraph{Other Machine Learning Methods for Solving Structure-Determination Inverse Problems} 
Parallel efforts in the deterministic AlphaFold 2 era, such as \citet{fadini2025alphafold}, explored optimizing latent coevolutionary embeddings to align predictions with experimental measurements. Other strategies integrate constraints through parameter-wise fine-tuning~\cite{xie2025integrating, stahl2023protein, zhang2025distance} or specialized architectures like ModelAngelo \citep{jamali2024automated}. 
Compared to these prior approaches, EmbedOpt is a flexible, inference-time method that exploits the rich generative capacity of state-of-the-art AlphaFold~3-style diffusion models without the computational cost of model fine-tuning. 

\paragraph{General Diffusion Inference-Time Steering.}
Beyond the biomolecular context, a broad class of methods steer diffusion models at inference time. At one end are guidance-based techniques such as classifier-free guidance~\citep{ho2022classifier} or DPS for inverse problems~\citep{chung2022diffusion}. At the opposite extreme are methods requiring additional training~\citep[e.g.][]{domingo2024adjoint}, offering stronger adaptation at higher computational cost. Between these lies a middle ground of training-free but more sophisticated inference-time methods, such as sequential Monte Carlo–based approaches~\citep{wu2023practical,ren2025driftlite}. 

A distinct line of work is prompt-tuning in text-to-image diffusion models, which steers generation by optimizing the conditioning prompt to maximize image-level rewards~\citep{hao2023optimizing,chung2023prompt}. While conceptually similar to EmbedOpt --- both optimize a conditional embedding --- these approaches typically rely on iterative embedding refinement or reinforcement learning loops, incurring substantially higher computational cost.

\section{Experiment Details}\label{app:sec:experiment}

\subsection{Post-processing via Energy Relaxation}\label{app:subsec:energy-relaxation} 
We apply a physics-based energy relaxation as a post-processing step to generated samples from the Protenix prior model, DPS and EmbedOpt, following prior work \citep{jumper2021highly, maddipatla2025inverse}. More specifically,  each sampled structure is completed with hydrogen atoms and locally relaxed via energy minimization under a classical AMBER force field \citep{ponder2003force}. 

Empirically, we find this step to be computationally lightweight. For most settings of EmbedOpt and DPS, applying energy relaxation produce physically plausible structures while largely preserving the experimental constraints, with only a small decrease in the corresponding reward value. However, for some settings, notably DPS with a large learning rate $\alpha$ (\Cref{app:fig:cryoem_failure_case}), energy relaxation is insufficient to repair the resulting globally broken geometry. We provide a detailed ablation study on the synthetic map benchmark below.

\subsubsection{Ablation Study: Effect of Energy Relaxation}
\label{app:sec:energy_relaxation}

We analyze the effect of energy relaxation on the synthetic map benchmark. We group steering methods, DPS and EmbedOpt, by learning rate (LR) regime: low ($\text{LR} < 0.01$, near-prior), moderate ($0.01 \le \text{LR} < 0.5$), and high ($\text{LR} \ge 0.5$. 
For each group we report median map correlation and MolProbity score before and after relaxation, together with the paired per-sample difference (relaxed $-$ unrelaxed). Brackets show the interquartile range (IQR). Results are summarized in  \Cref{tab:relax_cc,tab:relax_molprobity}. 

\begin{table}[!t]
    \centering
    \caption{\textbf{Ablation study: effect of energy relaxation on the synthetic map benchmark (metric: map correlation $\uparrow$).} Median [IQR] results reported. The \emph{Paired difference} column reports the per-sample difference (relaxed $-$ unrelaxed); negative values indicate that relaxation degrades map fit.}
    \label{tab:relax_cc}
    \vspace{0.5em}
    \begin{tabular}{lccc}
        \toprule
        Method & Unrelaxed & Relaxed & Paired difference \\
        \midrule
        Prior                                       & $0.68$ \,[$0.47$, $0.82$] & $0.64$ \,[$0.44$, $0.79$] & $-0.03$ \,[$-0.04$, $-0.02$] \\
        \midrule
        DPS, $\text{LR} < 0.01$                     & $0.73$ \,[$0.52$, $0.84$] & $0.70$ \,[$0.50$, $0.81$] & $-0.03$ \,[$-0.04$, $-0.02$] \\
        DPS, $0.01 \le \text{LR} < 0.5$             & $0.84$ \,[$0.69$, $0.92$] & $0.80$ \,[$0.66$, $0.89$] & $-0.03$ \,[$-0.04$, $-0.03$] \\
        DPS, $\text{LR} \ge 0.5$                    & $0.30$ \,[$0.00$, $0.49$] & $0.01$ \,[$-0.01$, $0.05$] & $-0.20$ \,[$-0.36$, $-0.01$] \\
        \midrule
        EmbedOpt, $\text{LR} < 0.01$                & $0.80$ \,[$0.65$, $0.90$] & $0.77$ \,[$0.61$, $0.86$] & $-0.03$ \,[$-0.04$, $-0.03$] \\
        EmbedOpt, $0.01 \le \text{LR} < 0.5$        & $0.95$ \,[$0.86$, $0.98$] & $0.89$ \,[$0.79$, $0.94$] & $-0.04$ \,[$-0.06$, $-0.03$] \\
        EmbedOpt, $\text{LR} \ge 0.5$               & $0.98$ \,[$0.93$, $0.99$] & $0.89$ \,[$0.73$, $0.94$] & $-0.09$ \,[$-0.20$, $-0.05$] \\
        \bottomrule
    \end{tabular}
\end{table}

\begin{table}[!t]
    \centering
    \caption{\textbf{Ablation study: effect of energy relaxation on the synthetic map benchmark  (metric: MolProbity score $\downarrow$).} Median [IQR] results reported. The \emph{Paired difference} column reports the per-sample difference (relaxed $-$ unrelaxed); negative values indicate that relaxation improves geometry.}
    \label{tab:relax_molprobity}
    \vspace{0.5em}
    \begin{tabular}{lccc}
        \toprule
        Method & Unrelaxed & Relaxed & Paired difference \\
        \midrule
        Prior                                       & $1.54$ \,[$1.42$, $1.69$] & $0.70$ \,[$0.60$, $0.83$] & $-0.85$ \,[$-0.99$, $-0.70$] \\
        \midrule
        DPS, $\text{LR} < 0.01$                     & $1.55$ \,[$1.39$, $1.71$] & $0.71$ \,[$0.61$, $0.83$] & $-0.84$ \,[$-0.98$, $-0.70$] \\
        DPS, $0.01 \le \text{LR} < 0.5$             & $1.52$ \,[$1.37$, $1.76$] & $0.71$ \,[$0.59$, $0.88$] & $-0.79$ \,[$-0.96$, $-0.65$] \\
        DPS, $\text{LR} \ge 0.5$                    & $4.48$ \,[$4.34$, $4.66$] & $3.27$ \,[$2.72$, $3.97$] & $-1.07$ \,[$-1.77$, $-0.41$] \\
        \midrule
        EmbedOpt, $\text{LR} < 0.01$                & $1.56$ \,[$1.40$, $1.71$] & $0.70$ \,[$0.61$, $0.83$] & $-0.84$ \,[$-1.00$, $-0.71$] \\
        EmbedOpt, $0.01 \le \text{LR} < 0.5$        & $1.92$ \,[$1.60$, $2.43$] & $0.77$ \,[$0.64$, $0.90$] & $-1.14$ \,[$-1.53$, $-0.90$] \\
        EmbedOpt, $\text{LR} \ge 0.5$               & $3.60$ \,[$2.77$, $4.35$] & $1.37$ \,[$0.89$, $2.26$] & $-1.95$ \,[$-2.29$, $-1.58$] \\
        \bottomrule
    \end{tabular}
\end{table}

 Overall, energy relaxation consistently improves physical plausibility  across most methods, while only slightly decreasing map correlation. The exception is DPS at high LR, which either produces structures that cannot be validated by Phenix or shows limited improvement after relaxation. We summarize the key findings below (medians reported).

\begin{itemize}
\item 
\textbf{Map correlation (higher is better):} Energy relaxation only slightly decreases the map correlation by approximately 0.03 - 0.04 for the prior and for DPS and EmbedOpt at low-to-moderate LRs. 
At high LRs $(\geq 0.5)$,  EmbedOpt shows a larger drop (-0.09) but still achieves strong final performance (0.89 vs. 0.64 for the prior). In contrast, DPS collapses (0.30 $\to$ 0.01), performing worse than the unguided prior even before relaxation.
\item 	\textbf{Physical plausibility (via Molprobity score, lower is better):}  
Energy relaxation consistently improves MolProbity scores by $\approx$ 0.8–0.9 for  DPS and EmbedOpt at low LRs, indicating a uniform effect in this regime. At moderate LRS, 
    DPS shows similar improvement (-0.79), while EmbedOpt benefits more substantially (-1.14), leading to better final scores (0.77 vs. 0.71 - 0.88).

At high LRs  ($\geq 0.5$), the behavior is completely different. EmbedOpt has large improvements from relaxation (-1.95), yielding a strong final score (1.37); In contrast, 
DPS with large LR produce 61.3\% degenerate structures that cannot be validated by Phenix before relaxation, and 5.6\% after. In comparison, EmbedOpt produces only 1 (0.2\%) degenerate sample at large LR. For the structure that can be validated, the Molprobity score improvement is limited (-1.07) and remains poor after relaxation (3.27), where improvements are computed on paired samples where both pre- and post-relaxation structures are valid.
\end{itemize}

\subsection{Cryo-EM Map Fitting Benchmark} \label{app:sec:cryoem}
\paragraph{Forward Model.} The differentiable forward model $V(\cdot)$ which renders a denoised structure $x_0$ to a map $V(x_0)$ is implemented with \texttt{SFC\_Torch} \cite{li2025sfcalculator} operating in its dedicated Cryo-EM mode. We first compute the Fourier-space structure factors $\mathbf{F}(\vec{h})$ representing the electrostatic potential by summing the electron scattering contributions of individual atoms:

\begin{equation}
    \mathbf{F}(\vec{h})=\sum_j O_j \cdot f_{\vec{h}, j} \cdot \operatorname{DWF}(\vec{h}) \cdot \exp \left[2 \pi i \vec{h} \cdot \vec{x}_j\right]
\end{equation}

where $j$ indexes the atoms, $O_j$ denotes occupancy (fixed at 1.0), and $x_j$ represents the fractional coordinates. The term $f_{\vec{h}, j}$ is the elastic atomic scattering factor for electrons, parameterized via a 5-Gaussian expansion \citep{peng1996robust, prince2004international}. This electron-specific parameterization is essential for accurately modeling Cryo-EM density and is the standard adopted by gold-standard validation suites such as \texttt{phenix.validation\_cryoem} \citep{afonine2018new}. Furthermore, explicitly modeling these Fourier-space intermediates allows for the mathematically rigorous application of the Debye-Waller factor $\mathrm{DWF}(\vec{h})$ (corresponding to a uniform atomic B-factor of $50 \text{\AA}^2$) and the precise application of a high-frequency resolution cutoff ($5.0 \text{\AA}$). Finally, an inverse FFT is applied to recover the real-space voxel intensities on the target grid.

\subsection{Distance-Constrained Structure Determination Benchmarks}

\paragraph{Forward Model.} We define the constraint set by identifying the atomic indices of the top $K=20$ residue pairs that exhibit the largest distance deviation between the unguided prior predictions and the ground truth PDB structure. The forward operator is then defined as the computation of the pairwise Euclidean distances for these specific atomic indices from the denoised structure $x_0$.

\begin{figure}[ht]
    \centering
    \includegraphics[width=1.0\linewidth]{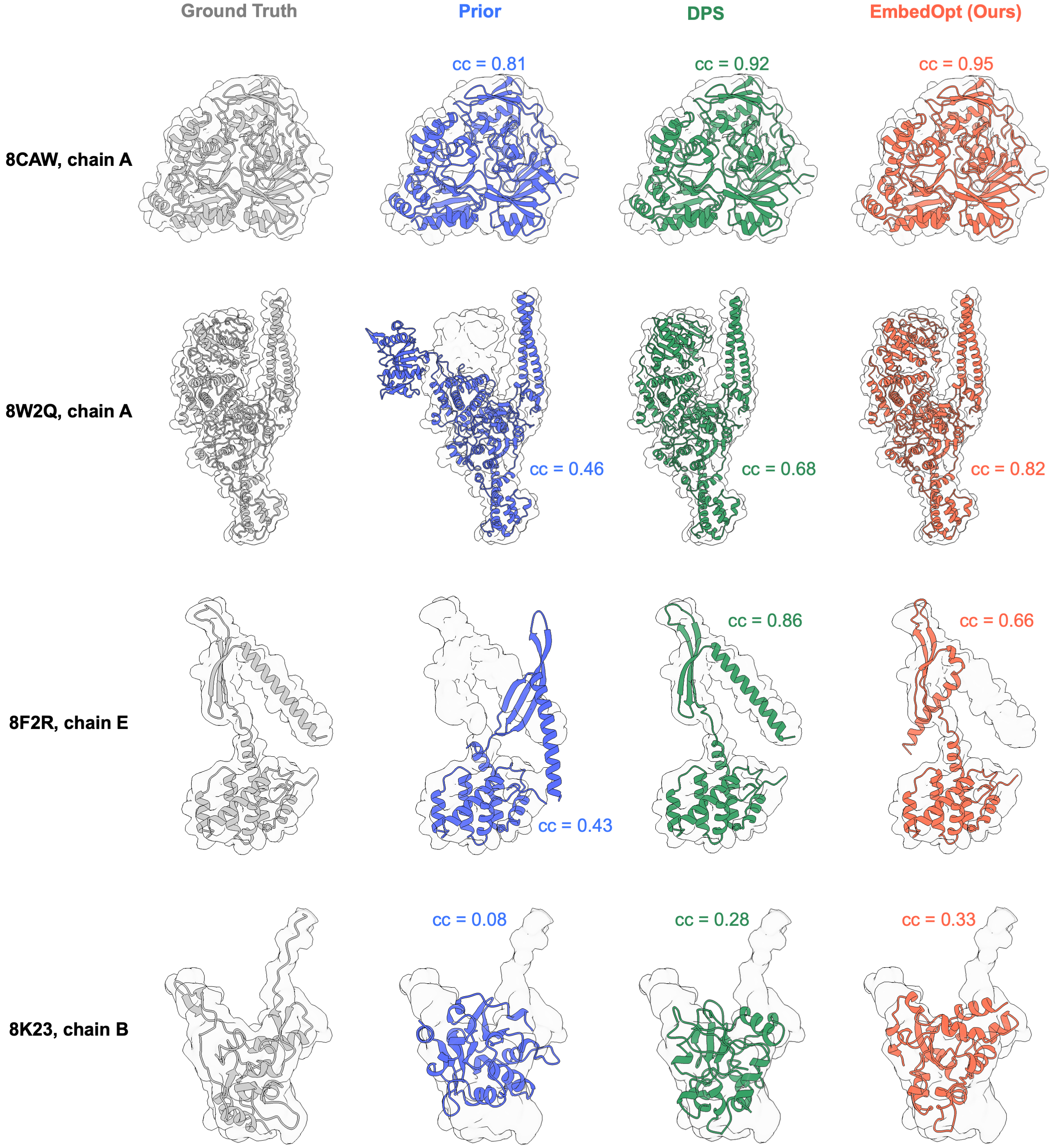}
    \caption{\textbf{Cryo-EM Map Fitting Benchmark: Sample Gallery of Representative Results across Test Systems (Hyperparameter-Tuned).} Structures display the best samples from each method following hyperparameter sweeping. Both methods perform robustly on targets where the unguided prior is already well-aligned with the target map (e.g., 8CAW). However, for targets requiring significant global conformational rearrangement (e.g., 8W2Q), EmbedOpt consistently achieves higher map correlation. We note 8F2R as the unique outlier where DPS significantly outperforms EmbedOpt, and 8K23 as a failure case where neither method successfully recovers the target structure.}
    \label{app:fig:cryoem_examples}
\end{figure}

\begin{figure}[ht]
    \centering
    \includegraphics[width=1.0\linewidth]{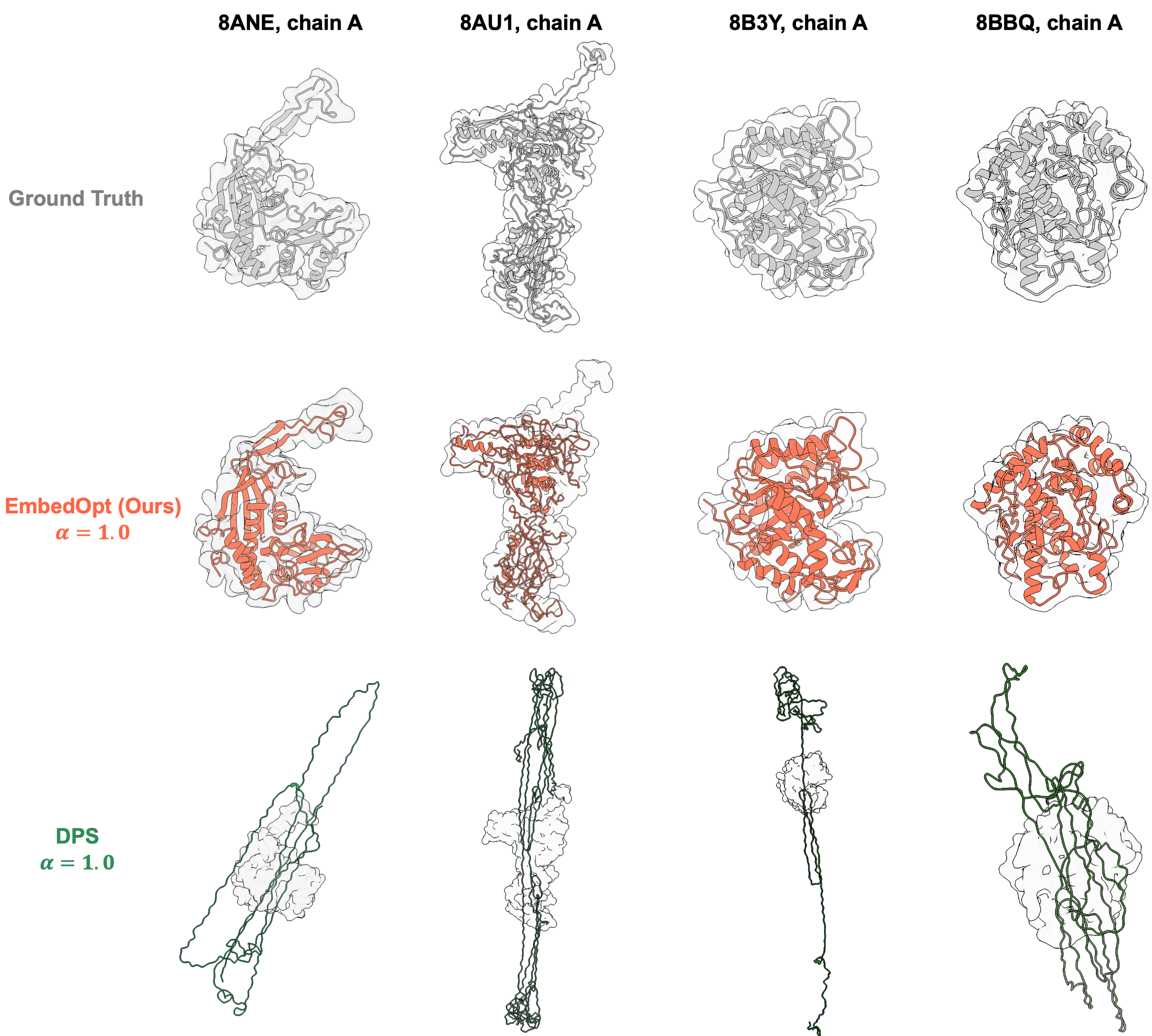}
    \caption{\textbf{Cryo-EM Map Fitting Benchmark: Sample Gallery of EmbedOpt and DPS Failure Modes under High Learning Rates.} We visualize the impact of high learning rates ($\alpha$=1.0) on generation quality for both methods. DPS (bottom) that directly steers  noisy coordinates can push trajectories off the data manifold, resulting in unphysical, unraveled structures that defy energy relaxation. In contrast, EmbedOpt (middle) remains structurally coherent even in this aggressive regime.}
    \label{app:fig:cryoem_failure_case}
\end{figure}

\begin{figure}[ht]
    \centering
    \includegraphics[width=0.4\linewidth]{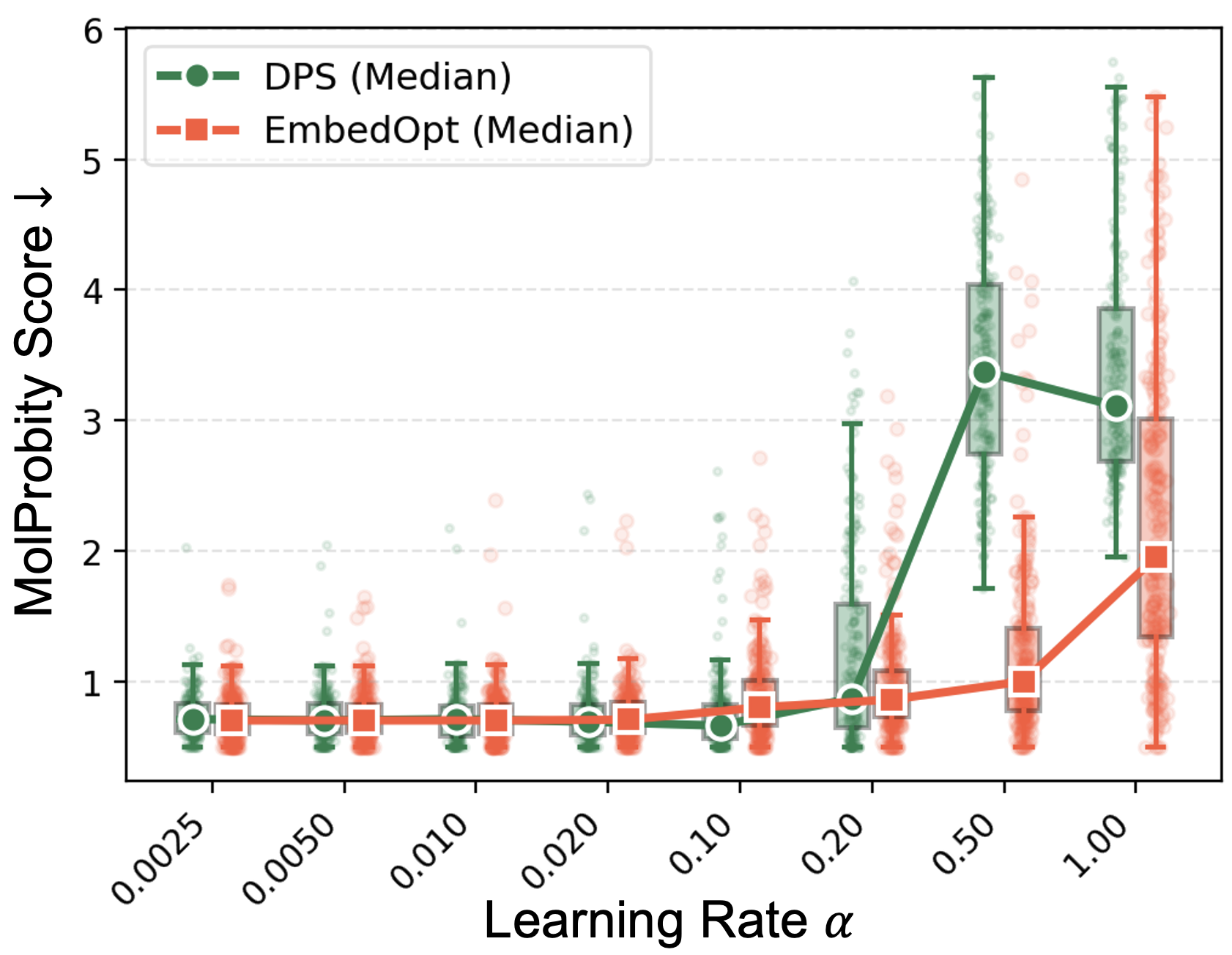}
    \caption{\textbf{Cryo-EM Map Fitting Benchmark: Stereochemical Quality Analysis.} Comparison of MolProbity scores (lower is better) across varying learning rates. EmbedOpt preserves structural validity even at high learning rates, whereas DPS suffers from severe geometric degradation when $\alpha >$ 0.1.}
    \label{app:fig:cryoem_molprobity_score}
\end{figure}

\begin{figure}[ht]
    \centering
    \includegraphics[width=0.97\linewidth]{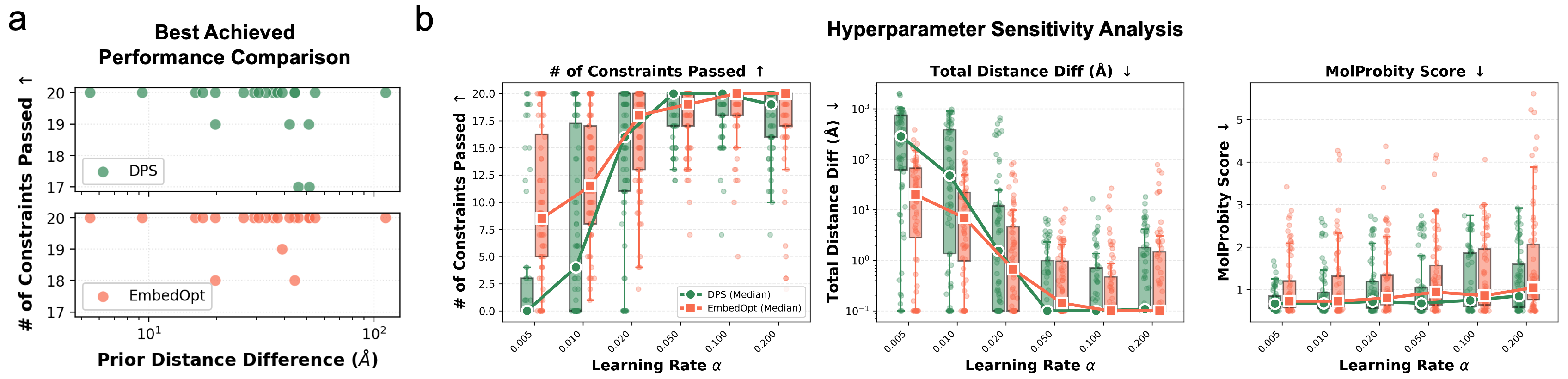}
    \caption{\textbf{Distance-Constrained Structure Determination Benchmark: Performance Analysis.} (a) \textbf{Best-Achieved Performance}: Comparison of constraint satisfaction on the $K=20$ constraint benchmark, where systems are ordered by task difficulty (initial deviation between prior and target distances per constraint). Unlike the unimodal likelihood in Cryo-EM task, sparse constraints allow both methods to achieve comparable peak performance, with both satisfying all constraints for the majority of targets. (b) \textbf{Hyperparameter Sensitivity}: Metric distributions across varying learning rates ($\alpha$) and fixed 200 steps for \# of Constraints Passed (left, higher is better), Total Distance Violation (middle, lower is better), and MolProbity Score (right, lower is better). EmbedOpt (orange) demonstrates better robustness, maintaining high constraint adherence and valid geometries (low MolProbity scores) across a broad learning rate spectrum. Box plots denote the interquartile range (IQR) with median bars; whiskers extend to $1.5\times$ IQR. Medians are highlighted. }
    \label{app:fig:af_distance_results}
\end{figure}

\subsection{Real-World Cryo-EM Map Fitting Benchmark}

\paragraph{CryoBoltz Baseline Protocol.}\label{app:par:cryoboltz_protocol} To provide a fair and direct comparison, we run CryoBoltz using their official codebase under default settings. In particular, CryoBoltz runs 200 diffusion steps while EmbedOpt uses 100 steps. We evaluate CryoBoltz on all six experimental cryo-EM targets using the same input data supplied to EmbedOpt: the identical cleaned density maps (cropping out background regions of the map under a threshold), protein sequences, and multiple sequence alignments. For each system, five independent guided structure predictions were generated using the official CryoBoltz inference pipeline. Validation metrics were evaluated on the raw CryoBoltz predictions (unrelaxed) as well as after post-prediction energy relaxation---structure preparation with PDBFixer followed by energy minimization with OpenMM---to assess the effect of geometry refinement and ensure the comparison reflects CryoBoltz at its best. All validation metrics (CC mask, MolProbity score, Ramachandran/rotamer statistics, and clashscore) were computed using the same Phenix-based pipeline (\texttt{phenix.validation\_cryoem} and \texttt{phenix.molprobity}) applied uniformly to both methods. Structural similarity to the deposited reference (TM-score, RMSD C$\alpha$, RMSD all-atom) was computed using Biotite after Kabsch superposition. This ensures that any observed differences in performance reflect the methods themselves rather than differences in input data, post-processing, or evaluation protocol. 

\begin{figure}[ht]
    \centering
    \includegraphics[width=1.0\linewidth]{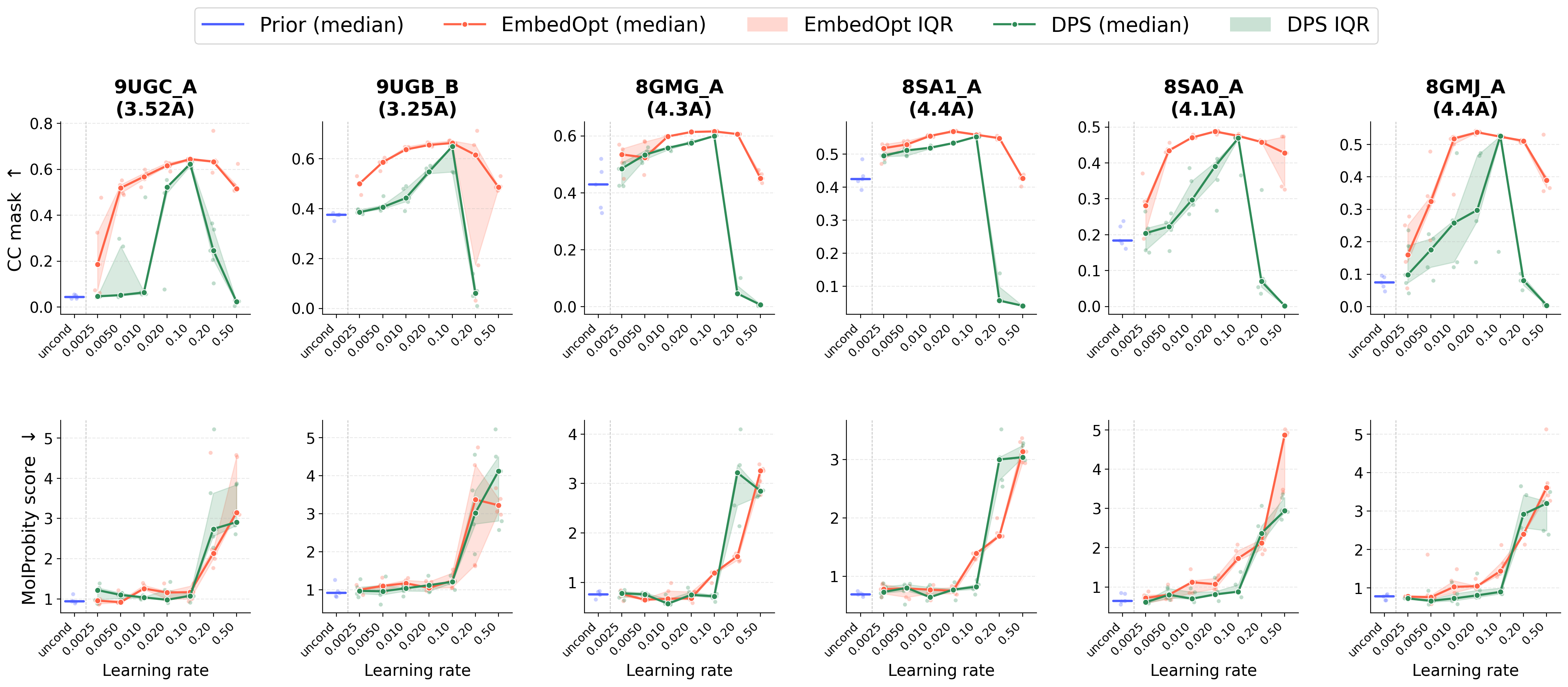}
    \caption{\textbf{Learning-Rate Sensitivity on Real Cryo-EM Targets.} Comparison of EmbedOpt and DPS across six experimental cryo-EM targets from the CryoBoltz benchmark. We sweep learning rates from 0.0025 to 0.5. \textbf{Top:} CC mask (higher is better) shows that EmbedOpt consistently improves map correlation relative to the unguided prior. While DPS shows a similar trend, it exhibits greater sensitivity to the guidance scale and deteriorates at high values. \textbf{Bottom:} MolProbity score (lower is better) indicates that EmbedOpt maintains stable, high-quality geometry across a broad range of learning rates, whereas DPS degrades significantly at higher values. Solid lines connect median values across seeds, shaded bands span the interquartile range, and individual seeds are shown as dots.}
    \label{app:fig:real_map_lr_sensitivity}
\end{figure}

\begin{figure}[!t]
    \centering
    \includegraphics[width=1.0\linewidth]{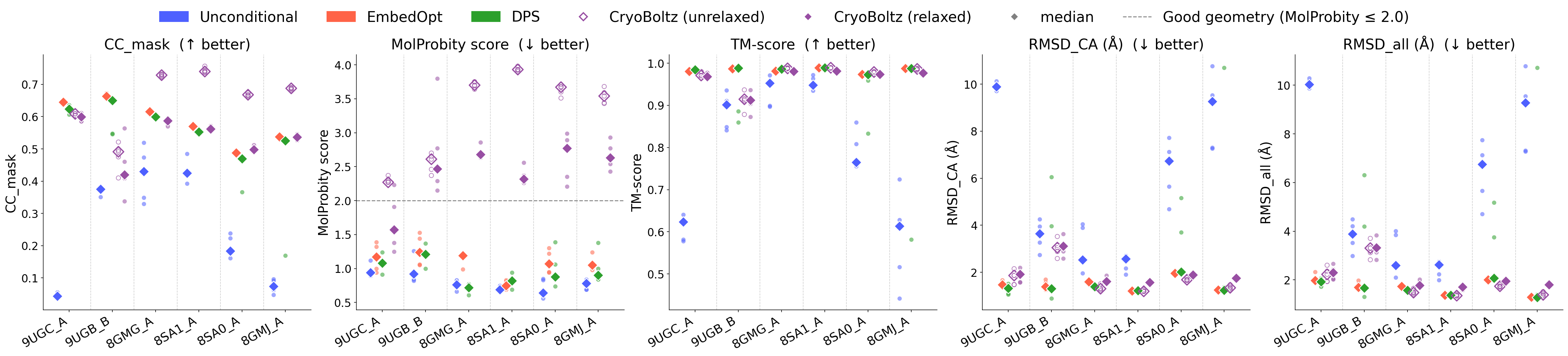}
    \caption{\textbf{Comprehensive Validation on Real Cryo-EM Targets, including DPS.} Evaluation of EmbedOpt against the unguided prior, DPS (within the Protenix prior), and the published CryoBoltz baseline across six experimental targets. EmbedOpt and DPS are evaluated using a consistent learning rate of $\alpha=0.1$, while CryoBoltz uses default published settings. Unrelaxed CryoBoltz predictions show pronounced reward-fitting artifacts—achieving high CC mask but with highly degraded geometry (MolProbity score $\gg 2.0$). Applying the same energy relaxation protocol used for EmbedOpt and DPS improves CryoBoltz's geometry but remains suboptimal, as severe stereochemical violations cannot be fully repaired by local optimization. EmbedOpt natively maintains excellent stereochemical quality (MolProbity score $<2.0$) at this learning rate. Comparing the relaxed structures, EmbedOpt outperforms both CryoBoltz and DPS in map correlation (CC mask) and global structural accuracy (TM-score, RMSD) on 5 out of 6 systems. EmbedOpt and DPS achieve comparable MolProbity scores. On the remaining system (8SA0), CryoBoltz achieves a marginally higher CC mask, but at the cost of significantly worse geometry. Individual seed predictions are shown as small dots, with medians indicated by diamond markers.}
    \label{app:fig:cryoboltz_comparison_si}
\end{figure}

\subsection{Extended Failure-Mode Analysis}\label{app:sec:failure_modes}

We expand on the four failure regimes summarized in the main-text discussion, illustrating each with concrete per-system evidence and outlining proposed remedies. The figures referenced below are introduced earlier in this appendix.

\paragraph{Manifold Coverage.} The cleanest illustration is target 8K23 from the synthetic cryo-EM benchmark (\Cref{app:fig:cryoem_examples}), where \emph{neither} EmbedOpt nor DPS recovers the deposited structure regardless of the learning rate used. Embedding optimization is fundamentally bounded by the support of the pretrained model: gradient ascent on the conditional embedding $c$ cannot place mass on conformations the model has never associated with sequences resembling the target. This is structurally analogous to the MSA-limited regime studied in \Cref{app:sec:ablation_msa}, where removing evolutionary information collapses both the prior and any inference-time steering of it. Both observations point to the same remedy: closing the manifold-coverage gap requires expanded pretraining---e.g.\ on conformationally diverse simulation data or harder negatives during MSA dropout---rather than any change to the inference-time procedure.

\paragraph{Local Optima.} On a small subset of synthetic cryo-EM targets, DPS attains a higher peak CC than EmbedOpt (e.g.\ 8F2R, \Cref{app:fig:cryoem_examples}). Inspection of the optimization traces suggests this is not a refutation of the embedding-vs-coordinate argument but a complementary observation: DPS's stochastic, noise-amplified coordinate updates act as a thermal perturbation that occasionally escapes narrow basins, while EmbedOpt's preconditioned, deterministic ascent commits decisively to the basin it enters. The same property that delivers EmbedOpt's headline robustness (smooth, monotone surrogate reward, \Cref{fig:distance_constraint}a) is also what penalizes it in this regime. This points concretely to hybrid two-stage schemes that we view as the most promising near-term extension: use EmbedOpt to shift the prior into a high-likelihood neighborhood of the target, then warm-start DPS (or inject controlled coordinate-space noise into EmbedOpt itself) to locally explore alternative basins. Such a pipeline would inherit EmbedOpt's stable global search and DPS's local mode-switching, while eliminating DPS's brittle cold-start behavior on the large prior--likelihood mismatches it currently struggles with (\Cref{fig:map_reward}a).

\paragraph{Reward-fitting at Extreme Learning Rates.} EmbedOpt's robustness band is wide but not infinite. At the upper end of the learning-rate sweeps (\Cref{fig:map_reward}c, \Cref{app:fig:cryoem_molprobity_score}), MolProbity scores eventually rise even for EmbedOpt, indicating that with a strong enough push the optimizer can be driven outside the model's structural manifold. The energy-relaxation step we apply uniformly to all methods (\Cref{app:subsec:energy-relaxation}) repairs mild violations but cannot fix coarse stereochemical errors once they accumulate. The CryoBoltz-relaxed comparison in \Cref{app:fig:cryoboltz_comparison_si} makes this concrete: severe pre-relaxation reward-fitting artifacts persist as elevated MolProbity even after the same OpenMM minimization protocol. Two principled remedies are worth pursuing: (i) explicit physics-aware regularization at inference time (e.g.\ adding a soft AMBER-energy penalty to the surrogate reward), and (ii) modifying the pretraining objective itself so that the conditioning module is explicitly regularized into a smooth latent space, e.g.\ via a variational bottleneck or contrastive structural objective. Both directions would extend the effective trust region beyond what post-hoc relaxation can achieve.


\subsection{Runtime and GPU Memory Profiling}
\label{app:sec:runtime_profiling}

\paragraph{Theoretical Analysis of Computational Complexity.} We first reason about the per-step cost of the prior model, DPS, and EmbedOpt in terms of forward and backward passes through the denoiser $\hat{x}_\theta$, which dominates the cost of all three samplers. The unguided prior requires a single forward pass per step, with no gradient computation. DPS adds a backward pass through $\hat{x}_\theta$ to compute $\nabla_{x_t} R(\hat{x}_0)$, for one forward and one backward pass per step. EmbedOpt requires the same forward-backward pair to compute the embedding gradient $\nabla_{c_t} R(\hat{x}_0)$, plus a second forward pass under the updated embedding to produce the sample. Using the standard rule of thumb that a backward pass costs roughly $2\times$ a forward pass in FLOPs, the per-step cost is approximately $C_F$ for the prior, $3 C_F$ for DPS, and $4 C_F$ for EmbedOpt, where $C_F$ is the runtime of one forward pass. This predicts EmbedOpt being $\sim$$4\times$ the prior and $\sim$$1.33\times$ DPS in runtime. For peak memory, both DPS and EmbedOpt are dominated by the activations stored from the single forward pass that is backpropagated through, so they should be comparable; the prior, which does no backward pass, is correspondingly lighter. On top of this shared baseline, however, EmbedOpt's pair-embedding gradient contributes an  $\mathcal{O}(N_{\text{res}}^2)$
 memory term that DPS's coordinate gradient ($\mathcal{O}(N_{\text{res}})$)
 lacks, predicting that EmbedOpt's memory overhead grows faster with system size.

\paragraph{Empirical Setup.} We profile EmbedOpt against the unguided prior and DPS on three targets spanning a range of system sizes: \texttt{8K23\_B} (177 residues) and \texttt{8P4K\_A} (599 residues) from the synthetic map benchmark, and \texttt{8GMG\_A} (1280 residues; P-glycoprotein / ABCB1, with the cropped EMD-40026 map at $4.3$\,\AA{} resolution) from the real map benchmark -- the largest system tested in this work. We additionally include the CryoBoltz baseline on \texttt{8GMG\_A} (its target in this work). All experiments use an NVIDIA H100 PCIe (80\,GB), a single sample per run ($N{=}1$), learning rate $0.1$ for EmbedOpt and DPS, and a $5$\,\AA{} map for the two smaller synthetic systems. Runtimes are averaged over 3 seeds and measured inside the diffusion loop using \texttt{torch.cuda.synchronize()}, excluding model loading, MSA processing,  the trunk forward pass and the post-processing energy relaxation --- these costs are shared across all methods, so the reported numbers reflect the \emph{marginal} cost of sampling rather than end-to-end wallclock. Peak memory is recorded with \texttt{torch.cuda.max\_memory\_allocated()} over the same loop.

\begin{table}[!t]
    \centering
    \resizebox{\textwidth}{!}{%
    \begin{tabular}{lcccccc}
        \toprule
        & \multicolumn{2}{c}{\texttt{8K23\_B} (177 res.)} & \multicolumn{2}{c}{\texttt{8P4K\_A} (599 res.)} & \multicolumn{2}{c}{\texttt{8GMG\_A} (1280 res.)} \\
        \cmidrule(lr){2-3} \cmidrule(lr){4-5} \cmidrule(lr){6-7}
        Method (base model) & Time (s) & Mem (GB) & Time (s) & Mem (GB) & Time (s) & Mem (GB) \\
        \midrule
        Prior (Protenix)     & $3.9 \pm 0.1$  & $2.10$ & $7.6 \pm 0.1$  & $5.33$ & $22.1 \pm 0.0$ & $10.10$ \\
        DPS (Protenix)       & $16.1 \pm 0.3$ & $1.97$ & $18.5 \pm 0.2$ & $5.93$ & $53.2 \pm 0.4$ & $23.90$ \\
        EmbedOpt (Protenix)  & $23.5 \pm 0.2$ & $2.10$ & $30.3 \pm 0.3$ & $7.58$ & $97.1 \pm 0.2$ & $30.16$ \\
        CryoBoltz (Boltz-1)  & /              & /      & /              & /      & $98.7 \pm 0.5$ & $27.12$ \\
        \bottomrule
    \end{tabular}%
    }
    \vspace{0.5em}
    \caption{\textbf{Per-trajectory diffusion-loop runtime and peak GPU memory of the unguided prior, DPS, EmbedOpt, and the CryoBoltz baseline on three targets spanning a range of system sizes.} Runtime is reported in seconds (mean $\pm$ standard deviation over 3 seeds); peak GPU memory is reported in gigabytes. Measured on an NVIDIA H100 PCIe; each run generates a single diffusion trajectory. Prior, DPS, and EmbedOpt use $100$ diffusion steps; CryoBoltz uses its native $200$ steps. CryoBoltz was profiled on \texttt{8GMG\_A} only; ``/'' denotes not run.}
    \label{tab:runtime_memory}
\end{table}

\paragraph{Empirical Results.} Results are summarized in \Cref{tab:runtime_memory} and broadly match the theoretical predictions. On the two larger targets, EmbedOpt is $4.0$--$4.4\times$ the prior and $1.6$--$1.8\times$ DPS, close to the predicted $4\times$ and $1.33\times$; on the smaller \texttt{8K23\_B}, the relative overheads are larger ($6.0\times$ and $4.1\times$ over the prior, respectively), reflecting fixed per-step overheads from gradient setup and reward evaluation that are amortized less effectively when $C_F$ is small. Memory follows the predicted scaling: at the smallest system, EmbedOpt and DPS are essentially indistinguishable from the prior, but the EmbedOpt-vs-DPS gap grows from $1.7$\,GB at $N_{\text{res}}{=}599$ to $6.3$\,GB at $N_{\text{res}}{=}1280$, consistent with the $\mathcal{O}(N_{\text{res}}^2)$ pair-embedding-gradient term. On the largest target, EmbedOpt's peak memory reaches $30.2$\,GB ($3.0\times$ the prior), still well within the capacity of a single H100. The CryoBoltz baseline on \texttt{8GMG\_A} reaches a per-step diffusion-loop cost comparable to EmbedOpt at twice the step count ($98.7$\,s for $200$ steps vs. $97.1$\,s for $100$ steps), so per step it is roughly half the cost of EmbedOpt; in peak memory it falls between DPS and EmbedOpt ($27.1$\,GB), since CryoBoltz backpropagates through the structure module but not the trunk.

\paragraph{End-to-end Wall-clock.} Beyond the marginal sampling cost, we also measure end-to-end wall-clock time on the same hardware --- the elapsed time of a complete inference job, including environment activation, Python startup, model loading, MSA processing, the trunk forward pass at $N_{\text{cycle}}{=}10$, the diffusion loop, and output writing (\Cref{tab:runtime_e2e}).  CryoBoltz's pipeline overhead (its own MSA featurization, model-weight cache check, Lightning Trainer/DataLoader setup, and manifest handling) is substantially heavier than ours, so although its diffusion loop is similar in cost to EmbedOpt's, its end-to-end wall-clock on \texttt{8GMG\_A} ($649$\,s) is $\sim$$2.7\times$ that of EmbedOpt ($244$\,s) and $\sim$$3.5\times$ that of the unguided prior ($188$\,s). For our methods, the fixed per-job overhead is approximately $150$\,s on the largest system and amortizes better as the diffusion loop lengthens (overhead share drops from $\sim$$88\%$ for the prior to $\sim$$60\%$ for EmbedOpt on \texttt{8GMG\_A}); for the smaller systems, end-to-end is dominated by overhead ($>80\%$) and is therefore comparable across the three methods.

\begin{table}[!t]
    \centering
    \begin{tabular}{lccc}
        \toprule
        & \multicolumn{3}{c}{End-to-End Time (s)} \\
        \cmidrule(lr){2-4}
        Method (base model) & \texttt{8K23\_B} (177 res.) & \texttt{8P4K\_A} (599 res.) & \texttt{8GMG\_A} (1280 res.) \\
        \midrule
        Prior (Protenix)    & $\phantom{0}96 \pm 1 $  & $122 \pm \phantom{0}9$  & $188 \pm \phantom{0}4$ \\
        DPS (Protenix)      & $112 \pm 2$            & $132 \pm 17$            & $205 \pm \phantom{0}7$ \\
        EmbedOpt (Protenix) & $118 \pm 2$            & $132 \pm \phantom{0}2$  & $244 \pm \phantom{0}1$ \\
        CryoBoltz (Boltz-1) & /                      & /                       & $649 \pm 84$ \\
        \bottomrule
    \end{tabular}
    \vspace{0.5em}
    \caption{\textbf{End-to-end wall-clock time per inference job} on the same NVIDIA H100 PCIe. Runtime is reported in seconds (mean $\pm$ standard deviation over 3 seeds). Includes environment activation, Python startup, checkpoint loading, MSA processing, trunk forward pass, diffusion loop, and output writing. Prior, DPS, and EmbedOpt use $100$ diffusion steps; CryoBoltz uses its native $200$ steps. CryoBoltz was profiled on \texttt{8GMG\_A} only; ``/'' denotes not run.}
    \label{tab:runtime_e2e}
\end{table} 

\subsection{Ablation Studies on Varying MSA Depths}
\label{app:sec:ablation_msa}

\begin{figure}[!t]
    \centering
    \begin{subfigure}[t]{0.48\linewidth}
        \centering
        \includegraphics[width=\linewidth]{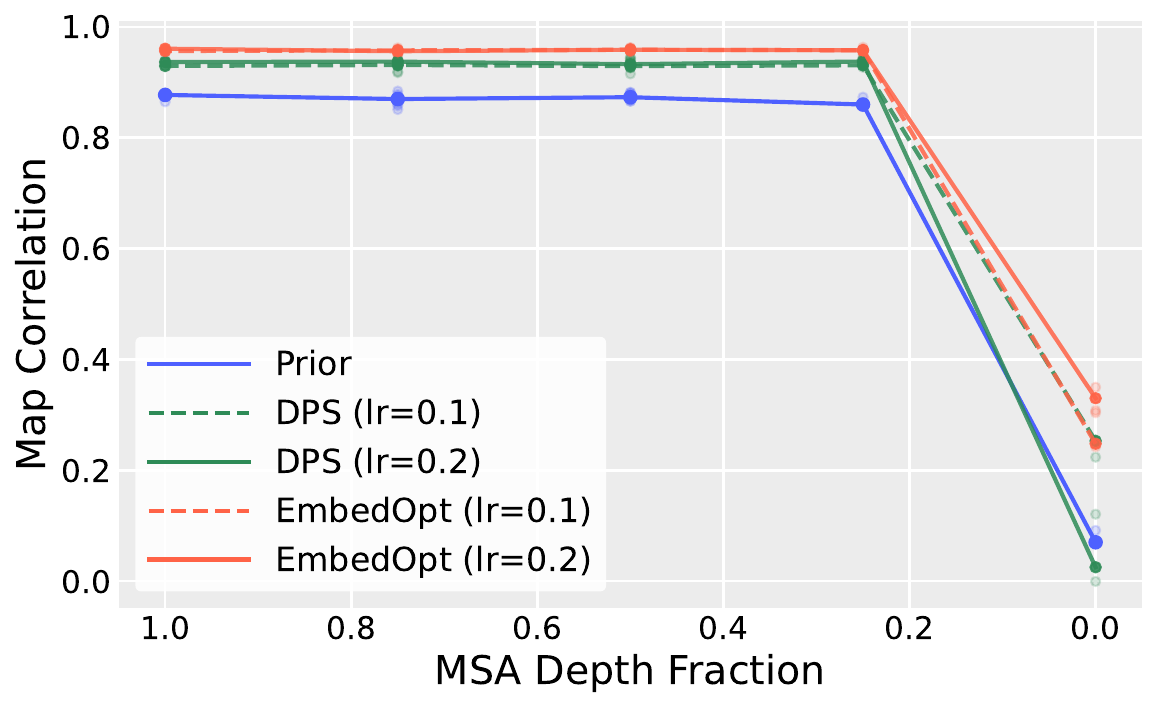}
        \caption{\texttt{8AHU\_A} --- Map Correlation ($\uparrow$)}
        \label{fig:msa_8AHU_CC}
    \end{subfigure}
    \hfill
    \begin{subfigure}[t]{0.48\linewidth}
        \centering
        \includegraphics[width=\linewidth]{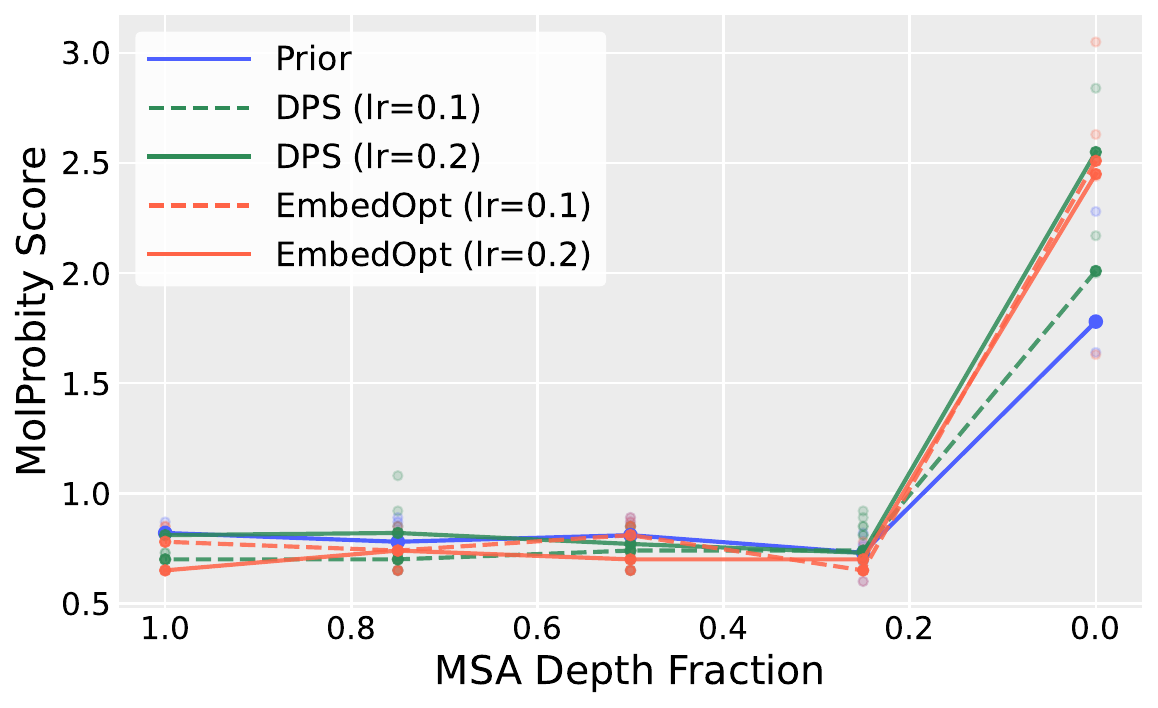}
        \caption{\texttt{8AHU\_A} --- MolProbity Score ($\downarrow$)}
        \label{fig:msa_8AHU_MP}
    \end{subfigure}

    \vspace{0.5em}

    \begin{subfigure}[t]{0.48\linewidth}
        \centering
        \includegraphics[width=\linewidth]{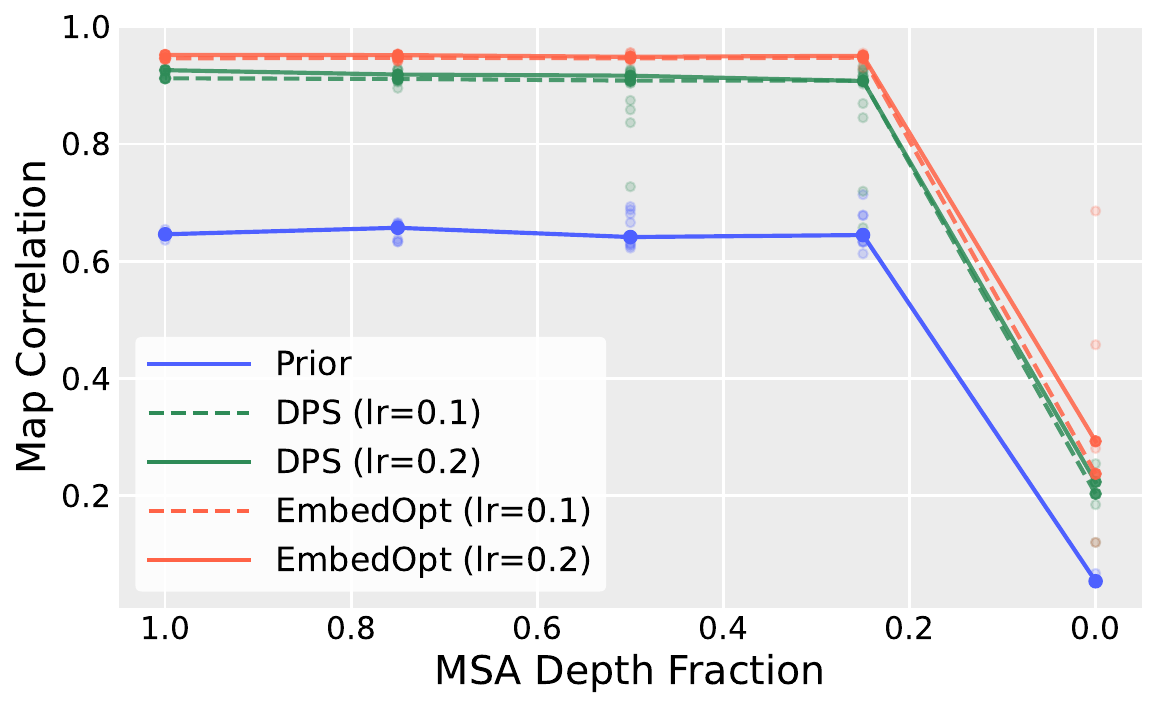}
        \caption{\texttt{8GXU\_A} --- Map Correlation ($\uparrow$)}
        \label{fig:msa_8GXU_CC}
    \end{subfigure}
    \hfill
    \begin{subfigure}[t]{0.48\linewidth}
        \centering
        \includegraphics[width=\linewidth]{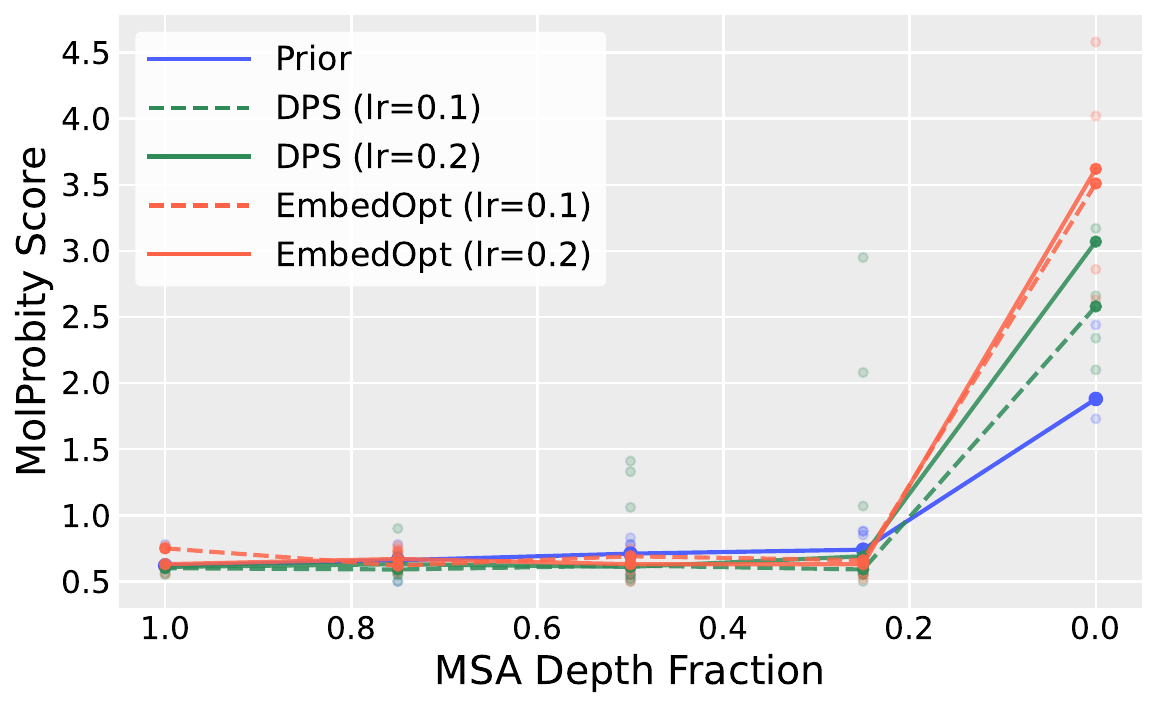}
        \caption{\texttt{8GXU\_A} --- MolProbity Score ($\downarrow$)}
        \label{fig:msa_8GXU_MP}
    \end{subfigure}

    \vspace{0.5em}

    \begin{subfigure}[t]{0.48\linewidth}
        \centering
        \includegraphics[width=\linewidth]{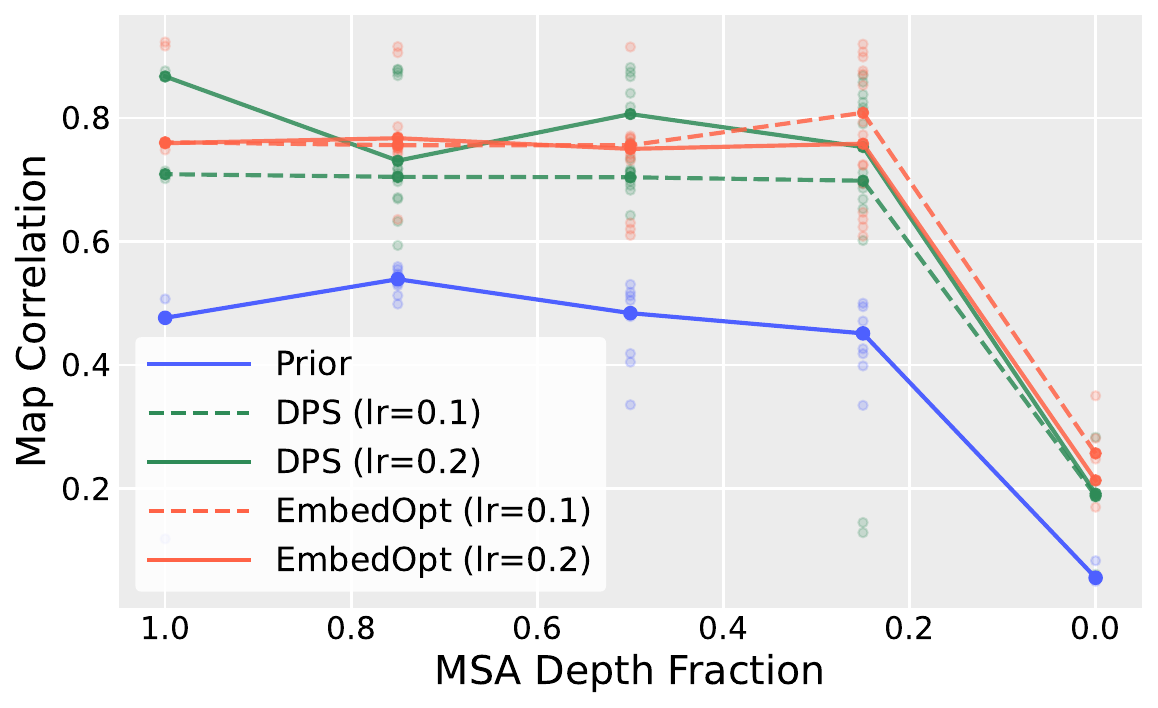}
        \caption{\texttt{8K9Z\_A} --- Map Correlation ($\uparrow$)}
        \label{fig:msa_8K9Z_CC}
    \end{subfigure}
    \hfill
    \begin{subfigure}[t]{0.48\linewidth}
        \centering
        \includegraphics[width=\linewidth]{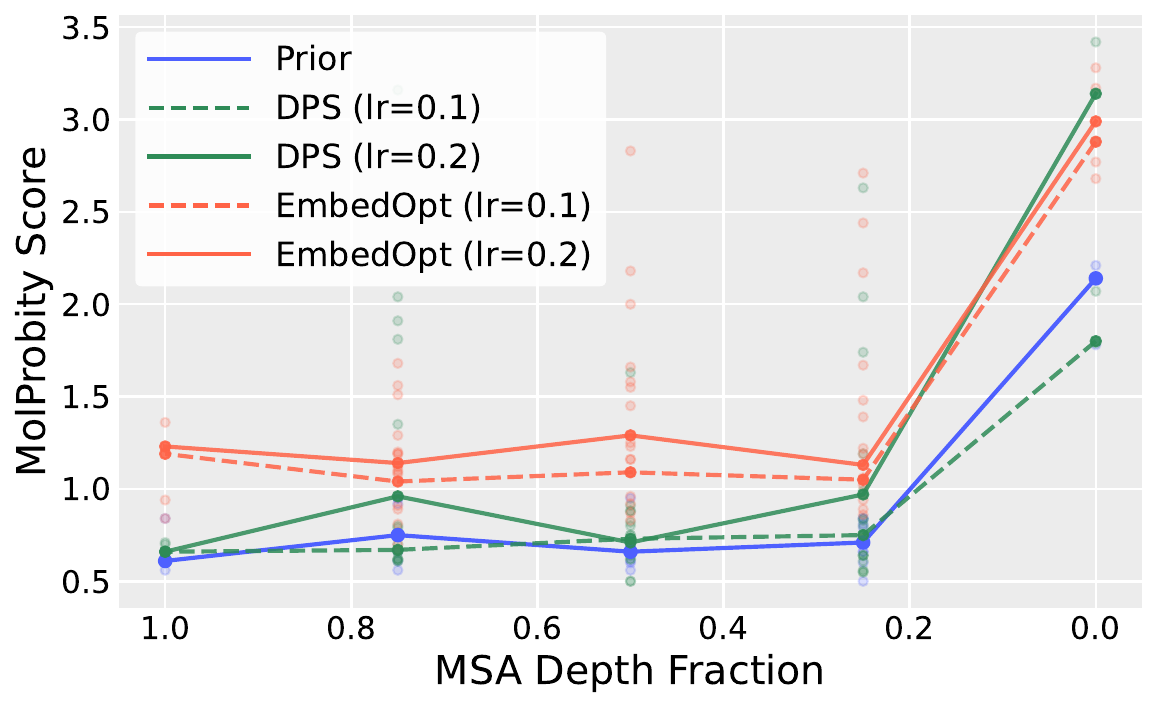}
        \caption{\texttt{8K9Z\_A} --- MolProbity Score ($\downarrow$)}
        \label{fig:msa_8K9Z_MP}
    \end{subfigure}

    \caption{\textbf{Effect of MSA depth on the unguided prior, DPS, and EmbedOpt across three targets of increasing difficulty.} Rows correspond to systems (\textbf{top}: \texttt{8AHU\_A}, easy; \textbf{middle}: \texttt{8GXU\_A}, moderate; \textbf{bottom}: \texttt{8K9Z\_A}, hard). The left column reports map-model correlation (CC, $\uparrow$ higher is better) and the right column reports MolProbity score ($\downarrow$ lower is better) as a function of MSA depth ($0\%$, $25\%$, $50\%$, $75\%$, $100\%$ of the full MSA, with random subsampling). Solid lines show the median across 3 MSA seeds $\times$ 3 diffusion seeds (9 runs per depth, except 3 at depths $0\%$ and $100\%$); semi-transparent dots show individual runs. Both steering methods, EmbedOpt and DPS,  consistently outperform the unguided prior between 25\% and 100\% MSA, but all methods collapse at 0\% MSA, signaling that evolutionary information is essential for the prior to provide a usable starting point. }
    \label{fig:msa_ablation}
\end{figure}

To probe the robustness of the unguided prior and the guided methods to a weakened evolutionary signal, we evaluate three targets in the synthetic map benchmark spanning a range of sizes and prior difficulties: \texttt{8AHU\_A} (283 residues; easy --- prior CC $\approx 0.87$), \texttt{8GXU\_A} (578 residues; moderate --- prior CC $\approx 0.65$), and \texttt{8K9Z\_A} (405 residues; hard --- prior CC $\approx 0.37$). For each system we sweep MSA depth $\in \{0\%, 25\%, 50\%, 75\%, 100\%\}$ (random MSA subsampling with 3 subsampling seeds $\times$ 3 diffusion seeds; at 0\% and 100\% MSA only 3 diffusion seeds  since no subsampling is required) and compare the unguided prior against two map-guided samplers, DPS and EmbedOpt, at learning rates $0.1$ and $0.2$.

Results are summarized in \Cref{fig:msa_ablation}. Across all three systems we observe a consistent trend. For all methods, both performance (map correlation) and geometric quality (MolProbity score) remain relatively stable across MSA depths from 25\%to 100\% and drop sharply at 0\% MSA. The steering methods, EmbedOpt and DPS, consistently outperform the unguided prior across the full range of MSA depths. At 0\% MSA, however, EmbedOpt and DPS provide only marginal improvements over the unguided prior, and all three remain far below their partial- or full-MSA performance. This result indicates that evolutionary information is essential for the prior to provide a usable starting point: steering can rescue a degraded prior, but cannot compensate for one that is entirely uninformed.

\begin{figure}[p]
    \centering
    \includegraphics[width=\linewidth, height=0.88\textheight, keepaspectratio]{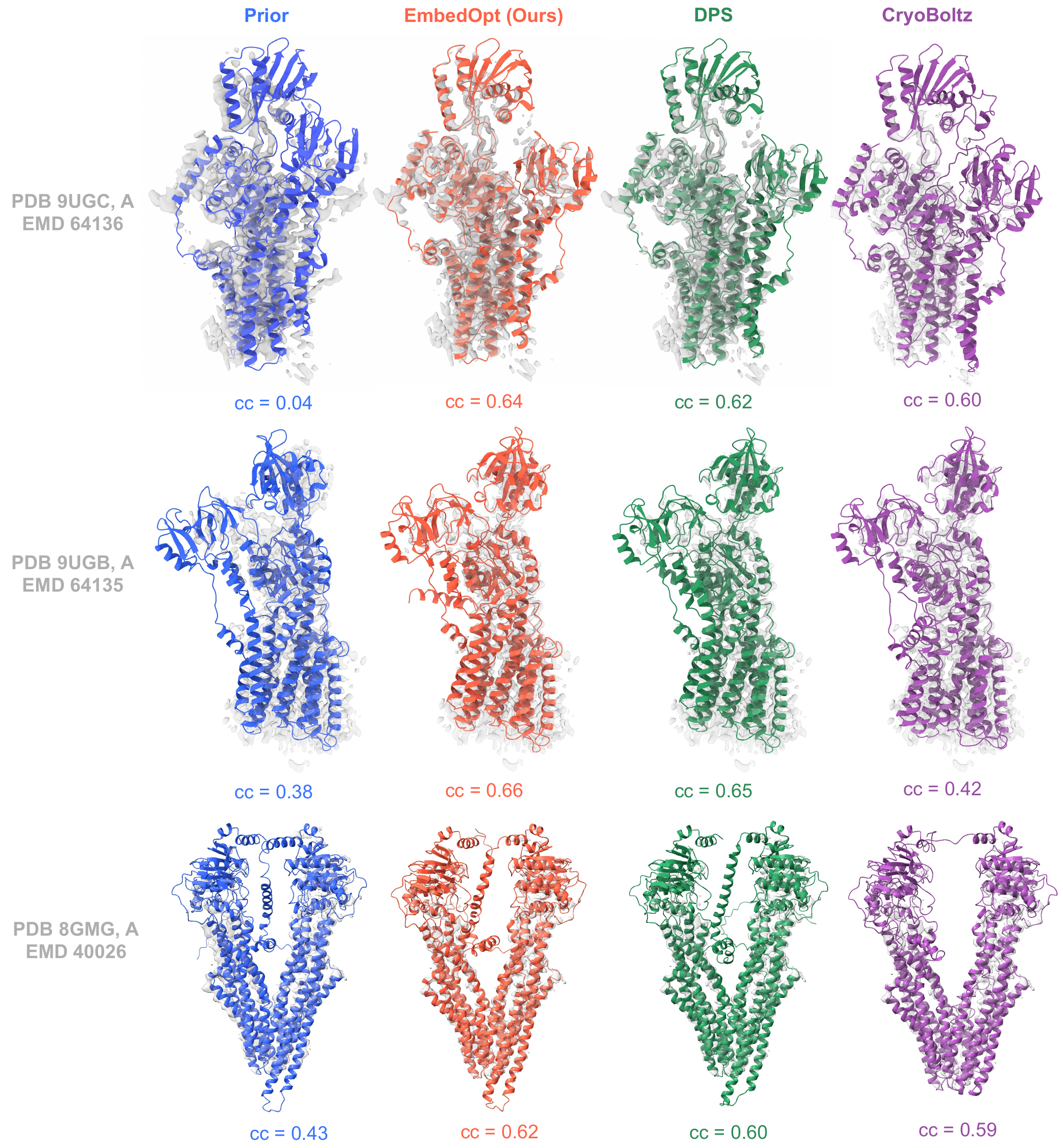}
    \caption{\textbf{Real Cryo-EM Map Fitting Benchmark: Visualization Gallery (Part 1 of 2).} Per-system visual comparison of predicted structures for the first half of the six CryoBoltz benchmark targets summarized quantitatively in \Cref{fig:real_map_cryoboltz}. For each system, we display the experimental cryo-EM density together with the predicted model from the median-CC seed of four methods: the unguided Protenix prior, DPS (within Protenix), EmbedOpt (within Protenix), and CryoBoltz (within Boltz-1). EmbedOpt uses a consistent learning rate of $\alpha = 0.1$; DPS uses the best learning rate per system from \Cref{app:fig:real_map_lr_sensitivity}; CryoBoltz uses default published settings. The remaining three targets are shown in \Cref{app:fig:real_cryoem_gallery_B}.}
    \label{app:fig:real_cryoem_gallery_A}
\end{figure}

\begin{figure}[p]
    \centering
    \includegraphics[width=\linewidth, height=0.88\textheight, keepaspectratio]{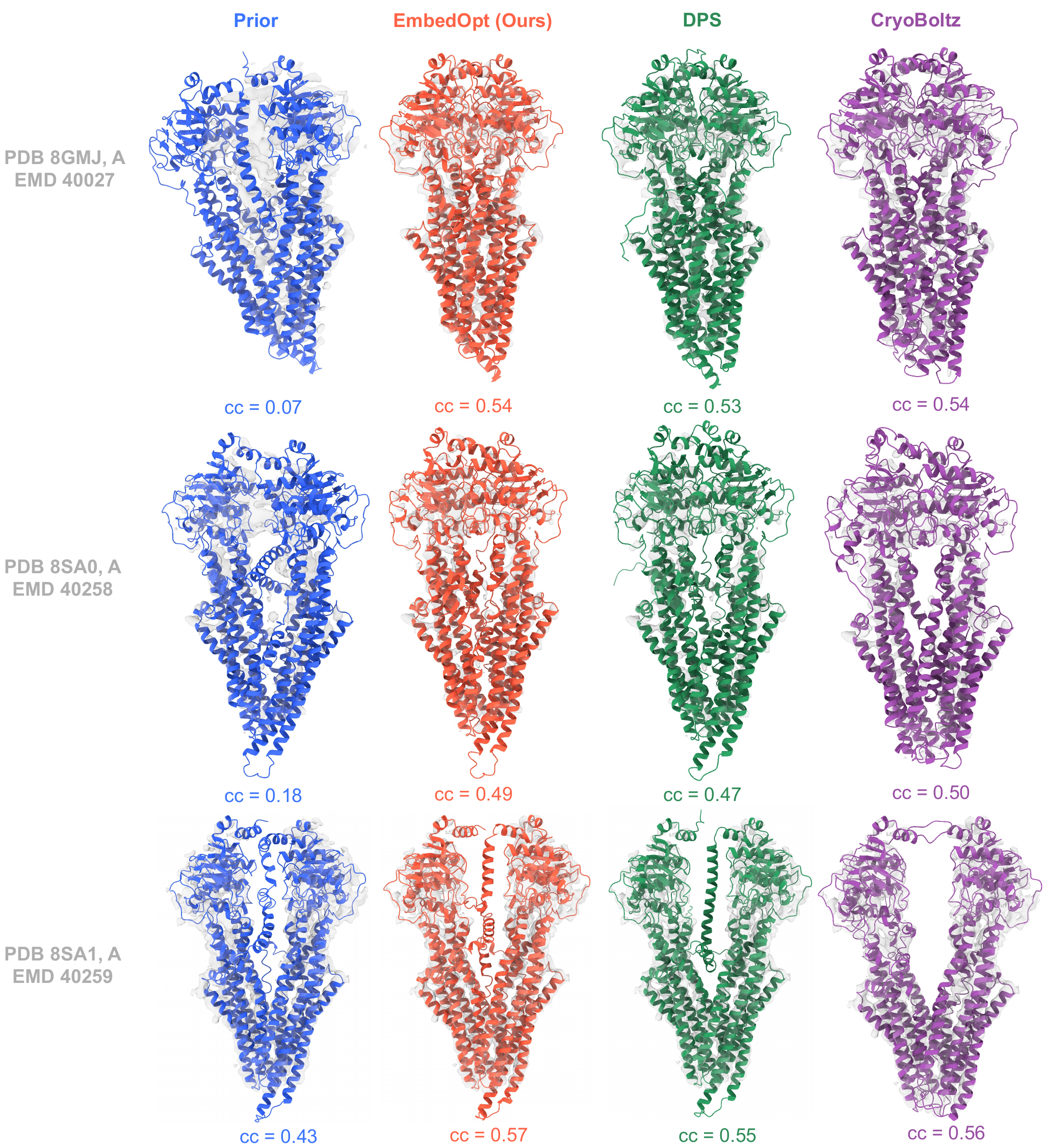}
    \caption{\textbf{Real Cryo-EM Map Fitting Benchmark: Visualization Gallery (Part 2 of 2).} Continuation of \Cref{app:fig:real_cryoem_gallery_A}, displaying the remaining three CryoBoltz benchmark targets under the same visualization protocol: experimental cryo-EM density overlaid with the median-CC predicted model from the unguided Protenix prior, DPS, EmbedOpt, and CryoBoltz. Method-specific learning-rate and configuration choices match those of \Cref{app:fig:real_cryoem_gallery_A}.}
    \label{app:fig:real_cryoem_gallery_B}
\end{figure}

\end{document}